\documentclass[sigconf]{aamas}

\usepackage{balance} % for balancing columns on the final page

\doi{VDEF3989}

\makeatletter
\gdef\@copyrightpermission{
  \begin{minipage}{0.2\columnwidth}
   \href{https://creativecommons.org/licenses/by/4.0/}{\includegraphics[width=0.90\textwidth]{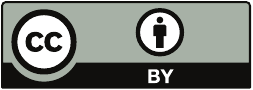}}
  \end{minipage}\hfill
  \begin{minipage}{0.8\columnwidth}
   \href{https://creativecommons.org/licenses/by/4.0/}{This work is licensed under a Creative Commons Attribution International 4.0 License.}
  \end{minipage}
  \vspace{5pt}
}
\makeatother

\setcopyright{ifaamas}
\acmConference[AAMAS '26]{Proc.\@ of the 25th International Conference
on Autonomous Agents and Multiagent Systems (AAMAS 2026)}{May 25 -- 29, 2026}
{Paphos, Cyprus}{C.~Amato, L.~Dennis, V.~Mascardi, J.~Thangarajah (eds.)}
\copyrightyear{2026}
\acmYear{2026}
\acmDOI{}
\acmPrice{}
\acmISBN{}

\acmSubmissionID{547}

\title[OPS]{When is Offline Policy Selection Sample Efficient for Reinforcement Learning?}

\author{Vincent Liu}\authornote{Work done at University of Alberta.}
\affiliation{
  \institution{RBC Borealis}
  \city{Toronto}
  \country{Canada}}
\email{vincent.liu@rbc.com}

\author{Prabhat Nagarajan}
\affiliation{
  \institution{University of Alberta}
  \city{Edmonton}
  \country{Canada}}
\email{nagarajan@ualberta.ca}

\author{Andrew Patterson}
\affiliation{
  \institution{Alberta Machine Intelligence Institute}
  \city{Edmonton}
  \country{Canada}}
\email{ap3@ualberta.ca}

\author{Martha White}
\affiliation{
  \institution{University of Alberta}
  \city{Edmonton}
  \country{Canada}}
\email{whitem@ualberta.ca}

\begin{abstract}
Offline reinforcement learning algorithms often require careful hyperparameter tuning. Before deployment, we need to select amongst a set of candidate policies. However, there is limited understanding about the fundamental limits of this offline policy selection (OPS) problem. In this work we provide clarity on when sample efficient OPS is possible, primarily by connecting OPS to off-policy policy evaluation (OPE) and Bellman error (BE) estimation. We first show a hardness result, that in the worst case, OPS is just as hard as OPE, by proving a reduction of OPE to OPS. As a result, no OPS method can be more sample efficient than OPE in the worst case.  We then connect BE estimation to the OPS problem, showing how BE can be used as a tool for OPS. While BE-based methods generally require stronger requirements than OPE, when those conditions are met they can be more sample efficient. Building on this insight, we propose a BE method for OPS, called Identifiable BE Selection (IBES), that has a straightforward method for selecting its own hyperparameters. We conclude with an empirical study comparing OPE and IBES, and by showing the difficulty of OPS on an offline Atari benchmark dataset.
\end{abstract}


\newcommand{\BibTeX}{\rm B\kern-.05em{\sc i\kern-.025em b}\kern-.08em\TeX}

\usepackage{amsmath,amsfonts,bbm}
\usepackage{thm-restate}
\usepackage{algorithm}
\usepackage[noend]{algpseudocode}
\usepackage{subcaption}

\newcommand{\bellman}{\mathcal{T}}
\newcommand{\norm}[1]{\| #1 \|}

\newcommand{\EE}{{\mathbb{E}}}

\newcommand{\bbP}{{\mathbb{P}}}

\newcommand{\ZZ}{{\mathbb{Z}}}
\newcommand{\II}{{\mathbb{I}}}

\newcommand{\cE}{\mathcal{E}}

\newcommand{\cL}{{\mathcal{L}}}
\newcommand{\cM}{{\mathcal{M}}}
\newcommand{\cI}{{\mathcal{I}}}
\newcommand{\cF}{\mathcal{F}}
\newcommand{\cG}{\mathcal{G}}
\newcommand{\cS}{\mathcal{S}}
\newcommand{\cA}{\mathcal{A}}
\newcommand{\cQ}{\mathcal{Q}}

\usepackage{amsthm}
\theoremstyle{definition} % no italic 

\newtheorem{definition}{Definition}

\newtheorem{innercustomlemma}{Lemma}

\newenvironment{mylemma}[1]
{\renewcommand\theinnercustomlemma{#1}\innercustomlemma}
{\endinnercustomlemma}

\def\Pr{\mathop{\rm Pr}\nolimits}

 \newcommand{\defeq}{:=}

\DeclareMathOperator*{\argmax}{arg\,max}
\DeclareMathOperator*{\argmin}{arg\,min}

\begin{document}

%%% The following commands remove the headers in your paper. For final 
%%% papers, these will be inserted during the pagination process.

\pagestyle{fancy}
\fancyhead{}

%%% The next command prints the information defined in the preamble.

\maketitle 

%%%%%%%%%%%%%%%%%%%%%%%%%%%%%%%%%%%%%%%%%%%%%%%%%%%%%%%%%%%%%%%%%%%%%%%%

\section{Introduction}

Offline reinforcement learning (RL)---learning a policy from a dataset---is useful for many real-world applications, where learning from online interaction may be expensive or dangerous \citep{levine2020offline}.
There have been significant advances in offline policy learning algorithms with demonstrations that performant policies can be learned purely from offline data \citep{agarwal2020optimistic}.
Successful use of these algorithms, however, requires careful hyperparameter selection \citep{wu2019behavior,gulcehre2020rl,kumar2021workflow}.

Despite the fact that hyperparameter selection is essential, it remains relatively unexplored for offline policy learning. One of the reasons is that it is inherently hard, as we show formally in this paper. Unlike supervised learning where validation performance can be used to select hyperparameters, in offline RL it is hard to validate a policy's performance.
In fact, it is well known that \textit{off-policy policy evaluation} (OPE)---estimating the performance of a policy from offline data---is hard \citep{wang2020statistical}.
To emphasize this fact, in Figure \ref{fig:correlation}, we visualize the poor correlation between the true performance and performance estimates using a standard OPE approach called Fitted Q Evaluation (FQE) \citep{le2019batch}.
The datasets are not adversarially designed; rather they are two Atari datasets that were designed to provide reasonable data for offline learning algorithms \citep{fu2020d4rl}.

\begin{figure}[t]
    \centering
    \includegraphics[scale=0.24]{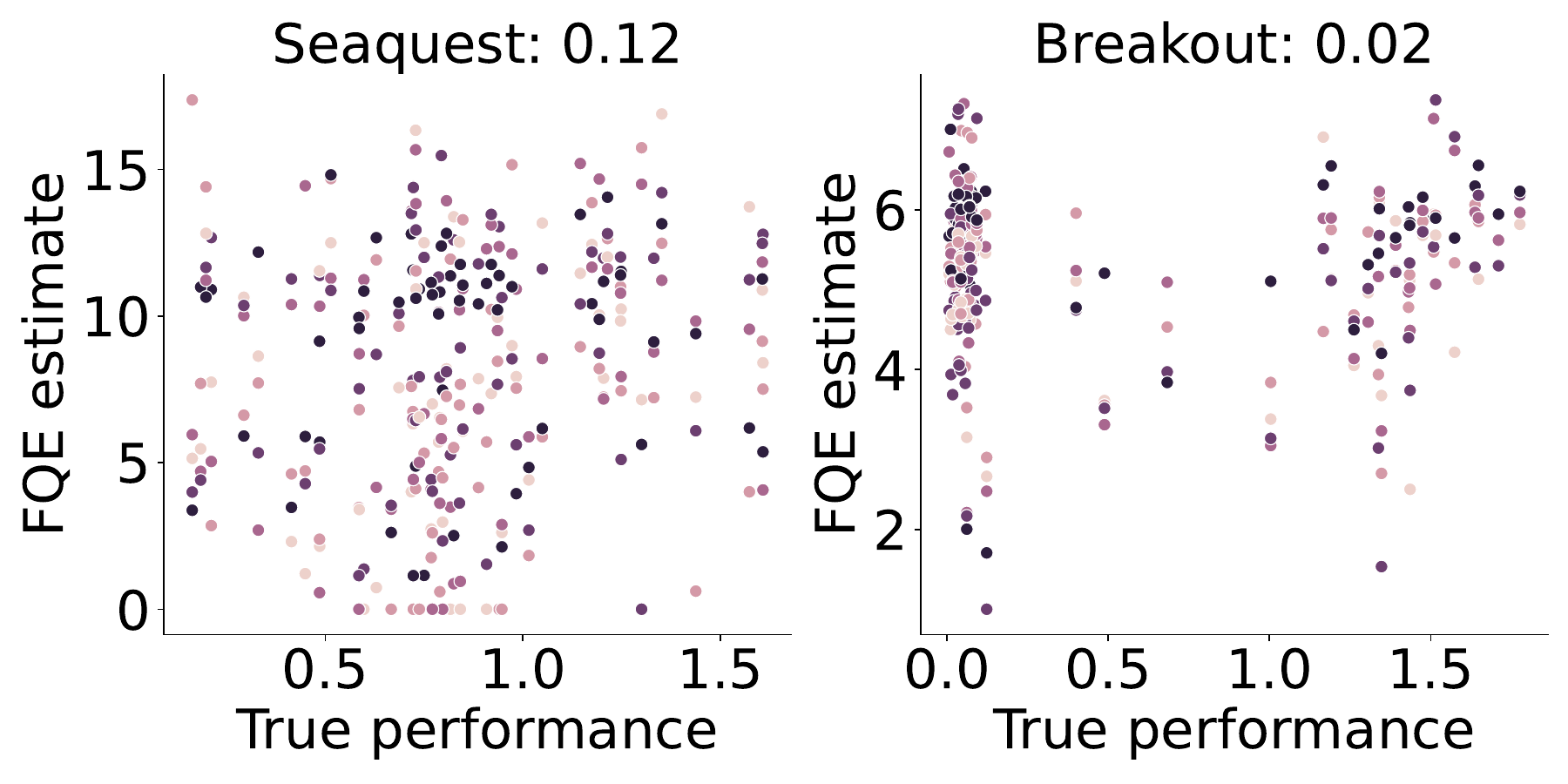}
    \caption{
        Correlation between true performance and estimated performance on Atari datasets. We generate 90 policies offline using the CQL algorithm with different choices for two hyperparameters: the number of training steps and the conservative parameter, and evaluate these policies using FQE on 5 different datasets. 
        Each point in the scatter plot corresponds to a (policy, evaluation dataset) pair. 
        The x-axis is the actual policy performance for that policy and y-axis is the estimated policy performance for that policy using that evaluation dataset. Colors represent different evaluation datasets. The Kendall rank correlation coefficient is shown in the title of the plot. If the FQE estimates accurately rank policies, we expect to see a strong linear relationship and a Kendall rank correlation coefficient close to $1$. Neither of these behaviors are seen here, and it is clear FQE does not provide an effective mechanism to rank policies. 
    }
    \Description{AAMAS.}
    \label{fig:correlation}
\end{figure}

Most offline RL papers sidestep the issue of hyperparameter selection and report results based on the best hyperparameter settings.
In other words, they tune their hyperparameters by checking the performance of the policy in the environment, which is against the spirit and formulation of offline RL.
Consider, for example, the Conservative Q-learning (CQL) algorithm \citep{kumar2020conservative}, used to obtain the policies for Figure \ref{fig:correlation}.
For their experiments, CQL is built on top of already-tuned online algorithms for their environments, which implicitly relies on tuning using the environment rather than just the offline dataset.
They then also directly tune the step size using the online performance of the offline learned policy in the environment.
Several papers \citep{kumar2019stabilizing,kostrikov2021offline} use the same methodology.
Some works focus on designing algorithms that are robust to a specific hyperparameter \citep{cheng2022adversarially}, or require only a small number of hyperparameters to be tuned \citep{fujimoto2021minimalist}. These approaches, however, still require some hyperparameters to be tuned; they were tuned using the performance in the environment.

There are a handful of works that use only the offline dataset to select hyperparameters. These can be categorized as papers focused on algorithms for hyperparameter selection in offline RL based on OPE \citep{paine2020hyperparameter,tang2021model,yang2022offline}, 
or papers that introduce offline RL algorithms and use sensible heuristics to select hyperparameters in their experiments \citep{yu2021combo,kumar2021workflow,trabucco2021conservative,qi2022data}. 
The heuristics are often tailored to the specific algorithm presented in these papers and do not come with performance guarantees. 

In general, there are as yet several open questions about the feasibility of hyperparameter selection in offline RL, both in theory and in practice. This is doubly true if we go beyond hyperparameter selection and consider selecting amongst different policy learning algorithms as well, each with their own hyperparameters.
In this paper, we consider this more general problem, called \emph{offline policy selection} (OPS), that can be used to select policy learning algorithms and their hyperparameters. Figure \ref{fig:ops_paradigm} depicts the full offline RL pipeline with OPS.

OPE is a general approach to OPS.
We can simply run OPE to produce a value estimate for each candidate policy and select the policy with the highest value estimate. 
Moreover, we might intuit that OPS may, at least in some cases, be easier than OPE.
OPS need only identify the best-performing policy, whereas OPE aims to produce an accurate value estimate of each policy, even if one candidate policy clearly dominates others. 
Therefore, the first question that needs to be answered is: 
\emph{Is OPS easier than OPE?}
Surprisingly, this basic question has not yet been explicitly answered.

Furthermore, we want to understand \emph{when, or even if, alternative approaches can outperform OPE for OPS}.
OPE might not be the ideal approach to OPS.
OPE requires a large number of samples to evaluate any given policy in the worst case \citep{wang2020statistical}. 
Moreover, OPE estimators do not resolve the OPS problem is because these OPE methods themselves have hyperparameters.
For example, the MAGIC estimator \citep{thomas2016data} and the clipped importance sampling (IS) estimator \citep{bottou2013counterfactual} are sensitive to their hyperparameters.
Even a variant of the IS estimators designed specifically for OPS \citep{yang2022offline} has hyperparameters. 
Other methods like FQE require learning a value function, and how to select the function class is still an open problem.

\begin{figure}[t]
    \centering
    \includegraphics[width=0.5\textwidth]{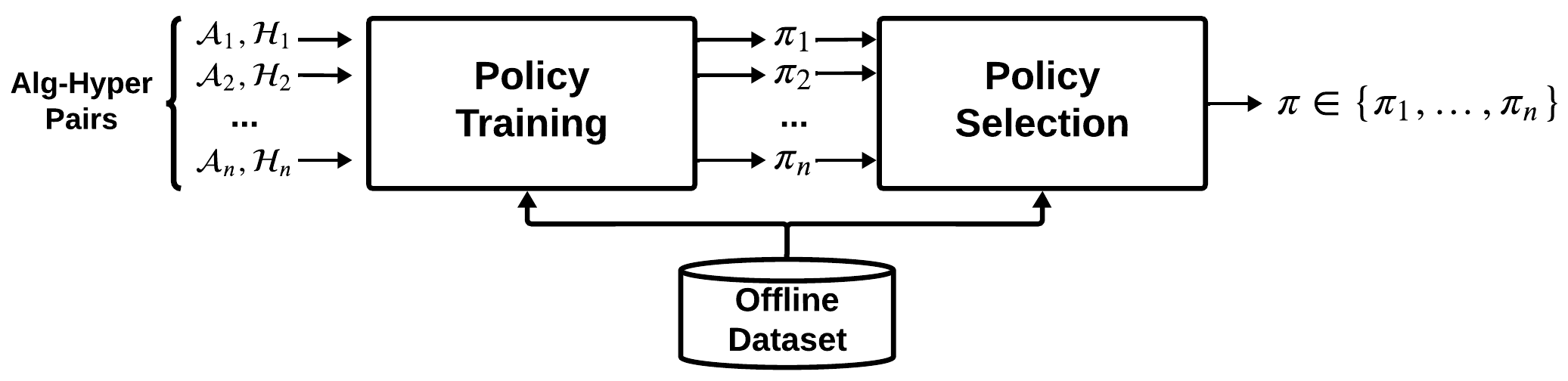}
    \caption{The offline RL pipeline with OPS. In the policy training phase, $n$ algorithm-hyperparameter pairs are trained on an offline dataset to produce $n$ candidate policies.
    An OPS algorithm then takes as input these $n$ policies, and again utilizing offline data (potentially a validation dataset), select a final policy.
    }
    \Description{AAMAS.}
    \label{fig:ops_paradigm}
\end{figure}

In this work, we make contributions towards answering these two questions. 
Our first contribution is to prove that OPS inherits the same hardness results as OPE, which has not been formally shown in the literature. 
This result implies no OPS approach can avoid exponential sample complexity in the worst case, and OPS requires additional assumptions to be sample efficient. 
Moreover, we connect the BE estimation problem to the OPS problem. We show that BE requires additional assumptions compared to OPE, but can be more efficient when the assumptions hold. 

Our second contribution is to propose a simple OPS algorithm based on Bellman errors (BE), called Identifiable BE selection (IBES), with a simple mechanism to select hyperparameters.
This is contrast to many OPE methods, even common approaches like FQE, for which hyperparameter selection is hard.
In an empirical study, we compare different BE methods with varying sample sizes and show that IBES is more sample efficient across multiple environments.

\section{Background} \label{sec:background}
In RL, the agent-environment interaction can be formalized as a finite horizon finite Markov decision process (MDP) $M=(\cS, \cA, H, s_0, Q)$. $\cS$ is a set of states with size $S=|\cS|$, and $\cA$ is a set of actions with size $A=|\cA|$, $H\in\ZZ^+$ is the horizon, and $s_0\in\cS$ is the initial state. We assume there is only a single initial state without loss of generality.
The reward $R$ and next state $S'$ are sampled from $Q$, that is, $(R,S')\sim Q(\cdot|s,a)$. 
We assume the reward is bounded in $[0, r_{max}]$ almost surely so the total return of each episode is bounded in $[0,V_{max}]$ almost surely with $V_{max}=Hr_{max}$.
The stochastic kernel $Q$ induces a transition probability $P:\cS\times\cA\to\Delta(\cS)$, and a mean reward function $r(s,a)$ which gives the mean reward when taking action $a$ in state $s$.

A non-stationary policy is a sequence of memoryless policies $(\pi_{0},\dots,\pi_{H-1})$ where $\pi_h:\cS\to\Delta(\cA)$. We assume that the set of states reachable at time step $h$, $\cS_h\subset\cS$, are disjoint, without loss of generality, because we can always define a new state space $\cS' = \cS \times [H]$, where $[H]$ denotes the set $[H]\defeq\{0,1,\dots,H-1\}$. Then, it is sufficient to consider stationary policies $\pi:\cS\to\Delta(\cA)$. 

Given a policy $\pi$, for any $h\in[H]$ and $(s,a)\in\cS\times\cA$, we define the value function and the action-value function as $v_h^\pi(s) \defeq \EE^\pi[\sum_{t=h}^{H-1} r(S_t,A_t)|S_h\!=\!s]$ and $q_h^\pi(s,a) \defeq \EE^\pi[\sum_{t=h}^{H-1} r(S_t,A_t)|S_h\!=\!s,A_h\!=\!a]$, respectively.
The expectation is with respect to $\bbP^\pi$, which is the probability measure on the random element $(S_0,A_0,R_0,\dots,R_{H-1})$ induced by the policy $\pi$. 
Similar to the stationary policy, we also use an action-value function $q$ to denote the sequence of action-value function $(q_{0},\dots,q_{H-1})$ since the set of states reachable at each time step are disjoint. That is, for all $h\in[H]$ and $s\in\cS_h$, $q(s,\cdot) = q_h(s,\cdot)$. 

The Bellman evaluation operator $\bellman^\pi$ is defined as
\begin{equation*}
    (\bellman^\pi q)(s,a) = r(s,a) + \sum_{s'\in\cS} P(s,a,s') \sum_{a'\in\cA} \pi(a'|s') q(s',a'),
\end{equation*}
with $(\bellman^\pi q)(s,\cdot) = r(s,\cdot)$ for $s\in\cS_{H-1}$. The Bellman optimality operator $(\bellman q)(s,a)$ is obtained if a greedy policy with respect to $q$ is used.
We use $J(\pi)$ to denote the value of the policy $\pi$, that is, the expected return from the initial state $J(\pi) = v^\pi_0(s_0)$. The optimal value function is defined by $v_h^*(s) \defeq \sup_{\pi} v_h^\pi(s)$. A policy $\pi$ is optimal if $J(\pi) = v^*_0(s_0)$. 

In the offline setting, we are given a fixed set of transitions $D$ with samples drawn from a data distribution $\mu$. We consider the setting where the data is collected by a behavior policy $\pi_b$ since the data collection scheme is more practical \citep{xiao2022curse} and is used to collect data for benchmark datasets \citep{fu2020d4rl}.
We use $d^\pi_h$ to denote the data distribution at the horizon $h$ by following the policy $\pi$, that is, $d^\pi_h(s,a) \defeq \bbP^{\pi}(S_h=s,A_h=a)$, and $\mu_h(s,a)\defeq \bbP^{\pi_b}(S_h=s,A_h=a)$.

\section{On the Sample Complexity of OPS} \label{sec:ops_sample_complexity}

We consider the offline policy selection (OPS) problem and off-policy policy evaluation (OPE) problem. We follow a similar notation used in \citet{xiao2022curse} to formally describe these problem settings. 
The OPS problem for a fixed number of episodes $n$ is given by the tuple $(\cS,\cA,H,\nu,n,\cI)$. $\cI$ is a set of instances of the form $(M,d_b,\Pi)$ where $M\in\cM(\cS,\cA,H,\nu)$ specifies an MDP with state space $\cS$, action space $\cA$, horizon $H$ and the initial state distribution $\nu$, $d_b$ is a distribution over a trajectory $(S_0,A_0,R_0,\dots,R_{H-1})$ by running the behavior policy $\pi_b$ on $M$, and $\Pi$ is a finite set of candidate policies. We consider the setting where $\Pi$ is finite and does not depend on $S$, $A$ or $H$.

An OPS algorithm takes as input a batch of data $D$, which contains $n$ trajectories, and a set of candidate policies $\Pi$, and outputs a policy $\pi\in\Pi$. 
% The aim is to find OPS algorithms that are guaranteed to return the best candidate policy with high probability.
The goal is to return the best policy with high probability.
% on every instance within a OPS problem:
\begin{definition}[$(\varepsilon,\delta)$-sound OPS algorithm]
    \label{def:ops}
    Given $\varepsilon>0$ and $\delta\in(0,1)$, an OPS algorithm $\cL$ is $(\varepsilon,\delta)$-sound on instance $(M,d_b,\Pi)$ if
    \begin{align*}
        \Pr_{D \sim d_b}(J_M(\cL(D,\Pi)) \geq J_M(\pi^\dagger) - \varepsilon) \geq 1-\delta,
    \end{align*}
    where $\pi^\dagger$ is the best policy in $\Pi$.
    We say an OPS algorithm $\cL$ is $(\varepsilon,\delta)$-sound on the problem $(\cS,\cA,H,\nu,n,\cI)$ if it is sound on any instance $(M,d_b,\Pi)\in\cI$. 
\end{definition}

Given a pair $(\varepsilon,\delta)$, the sample complexity of OPS is the smallest integer $n$ such that there exists a behavior policy $\pi_b$ and an OPS algorithm $\cL$ such that $\cL$ is $(\varepsilon,\delta)$-sound on the OPS problem $(\cS,\cA,H,\nu,n,\cI(\pi_b))$ where $\cI(\pi_b)$ denotes the set of instances with data distribution $d_b$. 
That is, if the sample complexity is lower-bounded by a number $\mathrm{N}_{OPS}$, then, for any behavior policy $\pi_b$, there exists an MDP $M$ and a set of candidate policies $\Pi$ such that any $(\varepsilon,\delta)$-sound OPS algorithm on $(M,d_b,\Pi)$ requires at least $\mathrm{N}_{OPS}$ episodes.   

Similarly, the OPE problem for a fixed number of episodes $n$ is given by $(\cS,\cA,H,\nu,n,\cI)$. $\cI$ is a set of instances of the form $(M,d_b,\pi)$ where $M$ and $d_b$ are defined as above, and $\pi$ is a target policy. 
An OPE algorithm takes as input a batch of data $D$ and a target policy $\pi$, and outputs an estimate of the policy value. The goal is to estimate the true value accurately. 

\begin{definition}[$(\varepsilon,\delta)$-sound OPE algorithm]
    Given $\varepsilon\in(0,V_{max}/2)$ and $\delta\in(0,1)$, an OPE algorithm $\cL$ is $(\varepsilon,\delta)$-sound on instance $(M,d_b,\pi)$ if 
    \begin{align*}
        \Pr_{D \sim d_b}(|\cL(D,\pi) - J_M(\pi)| \leq \varepsilon) \geq  1- \delta.
    \end{align*}
    We say an OPE algorithm $\cL$ is $(\varepsilon,\delta)$-sound on the problem $(\cS,\cA,H,\nu,n,\cI)$ if it is sound on any instance $(M,d_b,\pi)\in\cI$. Note that $\varepsilon$ should be less than $V_{max}/2$, otherwise the bound is trivial.
\end{definition}

\subsection{OPE as Subroutine for OPS}
It is obvious that a sound OPE algorithm can be used for OPS, since we can run the sound OPE algorithm to evaluate each candidate policy and select the policy with the highest estimate.
As a result, the sample complexity of OPS is upper-bounded by the sample complexity of OPE up to a logarithmic factor. For completeness, we state this formally in the theorem below. 

\begin{restatable}[Upper bound on sample complexity of OPS]{theorem}{opsupperbound}\label{thm:upper_bound}
    Given an MDP $M$, a data distribution $d_b$, and a set of policies  $\Pi$, suppose that, for any pair $(\varepsilon,\delta)$, there exists an $(\varepsilon,\delta)$-sound OPE algorithm $\cL$ on any OPE instance $I\in\{(M,d_b,\pi):\pi\in\Pi\}$ with a sample size at most $O(\mathrm{N}_{OPE}(S,A,H,1/\varepsilon,1/\delta))$.
    Then there exists an $(\varepsilon,\delta)$-sound OPS algorithm for the OPS problem instance $(M,d_b,\Pi)$ which requires at most $O(\mathrm{N}_{OPE}(S,A,H,2/\varepsilon,|\Pi|/\delta))$ episodes. 
\end{restatable}

In terms of the sample complexity, we have an extra $\sqrt{\log{|\Pi|}/n}$ term for OPS due to the union bound. For hyperparameter selection in practice, the size of the candidate set is often much smaller than $n$, so this extra term is negligible. However, if the set is too large, complexity regularization \citep{bartlett2002model} may need to be considered. In this paper, we consider a finite candidate set. 

\subsection{OPS is not Easier than OPE}
We have shown that OPS is sample efficient  when OPE is sample efficient. However, it remains unclear whether OPS can be sample efficient when OPE is not. 
In the following theorem, we lower bound the sample complexity of OPS by the sample complexity of OPE. 

\begin{restatable}[Lower bound on sample complexity of OPS]{theorem}{opslowerbound}           \label{thm:lower_bound}
    \sloppy
    Suppose for any data distribution $d_b$ and any pair $(\varepsilon,\delta)$ with $\varepsilon\in (0,V_{max}/2)$ and $\delta\in(0,1)$, there exists an MDP $M$ and a policy $\pi$ such that any $(\varepsilon,\delta)$-sound OPE algorithm requires at least $\Omega(\mathrm{N}_{OPE}(S,A,H,1/\varepsilon,1/\delta))$ episodes. 
    Then there exists an MDP $M'$ with $S'=S+2$, $H'=H+1$, and a set of candidate policies such that for any pair $(\varepsilon,\delta)$ with $\varepsilon\in (0,V_{max}/3)$ and $\delta\in(0,1/m)$ where $m\defeq\lceil\log(V_{max}/\varepsilon)\rceil\geq 1$, any $(\varepsilon,\delta)$-sound OPS algorithm also requires at least $\Omega(\mathrm{N}_{OPE}(S,A,H,3/2\varepsilon,1/m\delta))$ episodes. 
\end{restatable}

\begin{figure*}[t]
    \centering
    \begin{subfigure}{0.48\linewidth}
        \centering
        \includegraphics[scale=0.3]{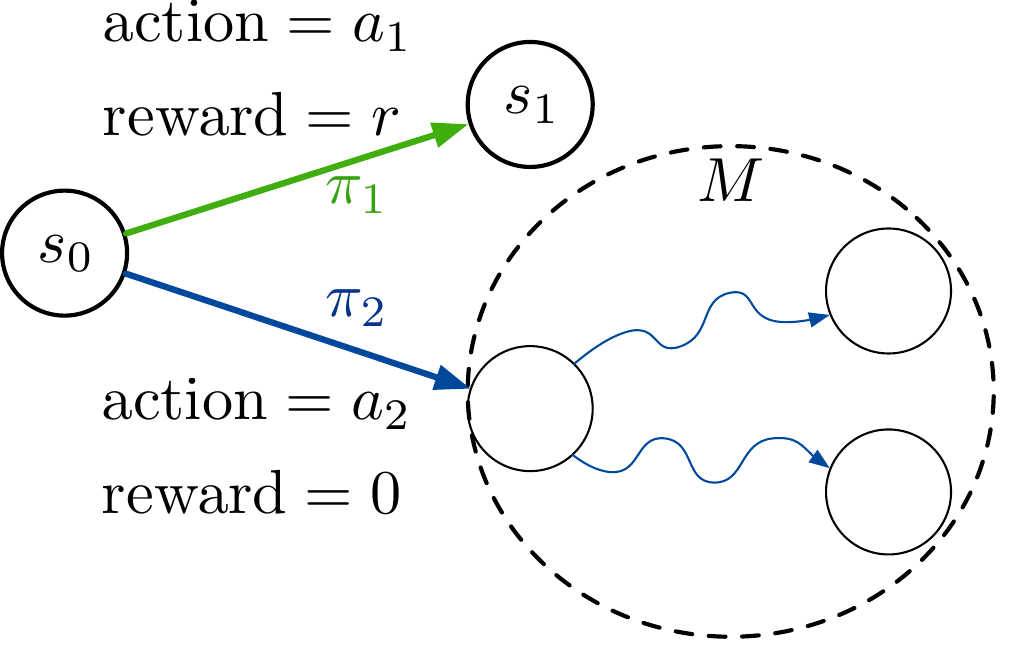}
        \caption{Constructing $M'$ and a candidate policy set}
        \Description{AAMAS.}
    \end{subfigure}
    \begin{subfigure}{0.48\linewidth}
        \centering
        \includegraphics[scale=0.3]{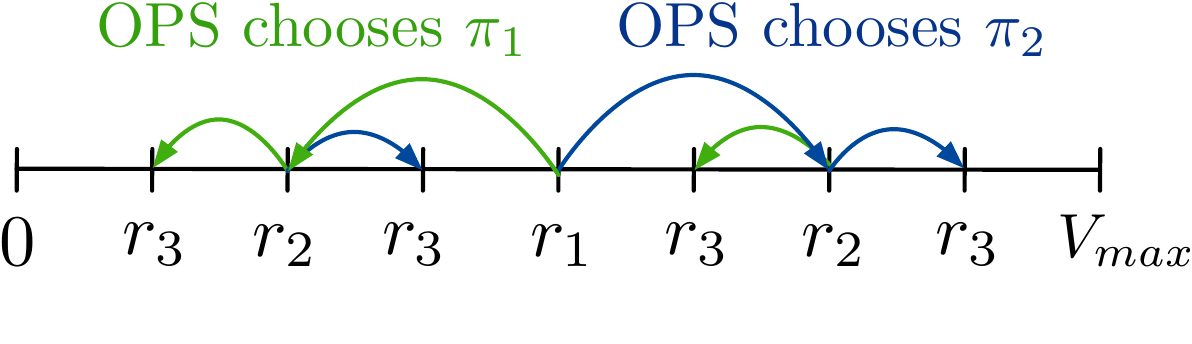}
        \caption{Searching for the true policy value}
        \Description{AAMAS.}
    \end{subfigure}
    \caption{
    Visual depiction of the reduction of OPE to OPS.
    Given a MDP $M$ and a target policy $\pi$, we can construct a new MDP $M'$ and two candidate policies $\{\pi_1,\pi_2\}$ for OPS, as shown in (a). The MDP construction was first mentioned in \citet{wang2020statistical}. $\pi_1$ chooses $a_1$ in $s_0$, which leads to a terminal state $s_1$, and can arbitrarily select actions in other states. $\pi_2$ chooses $a_2$ and is otherwise identical to the target policy $\pi$.
    Figure (b) describes the search procedure to find the policy value by calling the OPS subroutine.
    When the OPS query returns $\pi_1$, we follow the green arrow. 
    When the OPS query returns $\pi_2$, we follow the blue arrow.
     We can keep searching for the true policy value by setting $r$ for the OPS query, until the desired precision is reached. 
    }
    \label{fig:proof2}
\end{figure*}

The proof sketch is to construct an OPE algorithm that queries OPS as a subroutine, as demonstrated in Figure \ref{fig:proof2}. As a result, the sample complexity of OPS is lower bounded by the sample complexity of OPE. 
The proof can be found in the appendix.

There exist several hardness results for OPE in tabular settings and with linear function approximation \citep{yin2021optimal,wang2020statistical}. 
Theorem \ref{thm:lower_bound} implies that the same hardness results hold for OPS since lower bounds for OPE are also lower bounds for OPS. 
Theorem \ref{thm:lower_bound}, however, does not imply that OPS and OPE are always equally hard.
There are instances where OPS is easy but OPE is not. 
For example, when all policies in the candidate set all have the same value, any random policy selection is sound. 
However, OPE can still be difficult in such a scenario.

\subsection{Implication: We Need Assumptions for Sample Efficient OPS}
\label{sec:minimax}
In the previous section, we show that OPS and OPE problem are equally difficult in the worst case scenario. 
Based on the insights, we can show a lower bound on sample complexity of OPS for finite horizon finite MDPs.
Below, we present an exponential lower bound,
using the same lower bound construction from \citet{xiao2022curse} and Theorem  \ref{thm:lower_bound}.

\begin{restatable}[Lower bound on the sample complexity of OPS]{corollary}{explowerbound}      
    For any positive integers $S,A,H$ with $S>2H$ and a pair $(\varepsilon,\delta)$ with $0<\varepsilon\leq \sqrt{1/8}$, $\delta\in(0,1)$, any $(\varepsilon,\delta)$-sound OPS algorithm needs at least $\tilde\Omega(A^{H-1}/\varepsilon^2)$ episodes. 
    \label{corollary:ope_lower_bound}
\end{restatable}
 
These result suggests that we need to consider additional assumptions on the environment, the data distribution, or the candidate set to achieve sample efficient OPS.
Without such assumptions, any sound OPS algorithm would suffer exponential sample complexity in the worst case. 
One direction is to consider when OPE can be sample efficient, since we can leverage Theorem \ref{thm:upper_bound} to inherit the sample complexity result from OPE to OPS.
There is a wealth of literature on OPE, in particular, FQE and MIS (or DICE) methods have been shown to be effective for OPS empirically \citep{paine2020hyperparameter,yang2022offline}. FQE has also been shown to be sample efficient under a standard data coverage assumption, that is, data coverage for all the candidate policies, and a function approximation assumption.

\section{Bellman Error Selection for OPS}
\label{sec:BE}
The natural alternative to OPE is to use Bellman errors (BE) or value error to select the best value functions \citep{farahmand2011model}, however, there is mixed evidence on whether such approaches are effective.
\citet{tang2021model} empirically show that selecting the candidate value function with the smallest TD errors perform poorly. Similarly, \citet{paine2020hyperparameter} present positive experimental results using FQE for OPS but conclude that the use of TD errors was ineffective. Other works, however, have developed new OPS algorithms that rely on the use of BE.
\citet{zhang2021towards} propose a value-function selection algorithm called BVFT, which computes the (empirical) projected BE for each pair of candidate value functions. 
\citet{lee2022oracle} provide a method for selecting the best function class from a nested set of function classes for Fitted Q-Iteration. 
These theoretically-sound methods, however, either rely on strong assumptions or are applicable only in specialized settings. It is unclear whether these methods are better than OPE. 

In this section, we first show conditions when BE can be used OPS, which are generally stricter than the conditions needed for OPE methods. This suggests that, in theory, there is no benefit of using BE for OPS.
% On the other hand, BE offers a significant advantage in practice---we can have a simple way to select the hyperparameters. 
% As a demonstration, we propose a BE selection method that has a simple method to select its own hyperparameters using cross-validation.
On the other hand, BE offers a significant advantage in practice---we have a simple way to do model selection. 
As a demonstration, we propose a BE selection method that uses cross-validation for model selection.

\subsection{BE Selection Problem}
The selection problem using BE is slightly different from the OPS problem we defined in the previous section. 
Suppose we are given a set of candidate value functions $\cQ \defeq \{q_1,\dots,q_K\}$ and let $\Pi=\{\pi_1,\dots,\pi_K\}$ be the set of corresponding (greedy) policies, a common strategy is to select the action value function with the smallest BE \citep{farahmand2011model}.
In the finite horizon setting, the Bellman error with respect to $q_i$ is defined as 
\begin{align*}
    \mathcal{E}(q_i) \defeq \frac{1}{H} \sum_{h=0}^{H-1} \norm{q_{i} - \bellman q_{i}}_{2,\mu_h}^2
\end{align*} 
where $\norm{q}_{p,\mu_h} \defeq (\sum_{(s,a)\in\cS_h\times\cA} \mu_h(s,a) |q(s,a)|^p)^{1/p}$.
We define $(\varepsilon,\delta)$-sound \emph{BE selection} in the following. 

\begin{definition}[$(\varepsilon,\delta)$-sound BE selection]
    Given a set of candidate value functions $\cQ$, $\varepsilon>0$ and $\delta\in(0,1)$, an BE selection algorithm $\cL$, which takes $D,\cQ$ as input and outputs $q\in\cQ$, is $(\varepsilon,\delta)$-sound on $\cQ$ if $\mathcal{E}(\cL(D,\cQ)) \leq \min_{i=1,\dots,|\cQ|} \mathcal{E}(q_i) + \varepsilon$
    with probability at least $1-\delta$. 
\end{definition}

In order to connect the BE selection problem to the OPS problem in Definition~\ref{def:ops}, we need the following assumptions:
\begin{enumerate}
    \item (date coverage)\\ $\exists C$ such that $\forall\pi\in\Pi \cup \{\pi^*\}$,  $\underset{h\in [H]}{\max}  \underset{s\in\cS_h,a\in\cA_h}{\max} \frac{d_h^\pi(s,a)}{\mu_h(s,a)} \leq C$,
    \item (suboptimality of the candidate set)  $\underset{q\in\cQ}{\min} \ \mathcal{E}(q) \leq \varepsilon_{sub}$.
\end{enumerate}

With these assumptions, an $(\varepsilon_{est},\delta)$-sound BE selection algorithm $\cL$ on $\cQ$ is $(2 H \sqrt{C(\varepsilon_{sub}+\varepsilon_{est})},\delta)$-sound for the OPS problem on $\Pi$.
The bound can be easily derived based on existing results \citep{duan2021risk,xie2020q}.
We present the result here for clarity and include the complete theorem and the proof in the appendix.

Compared to OPE methods such as FQE, we need data coverage for both the candidate policies $\Pi$ and an optimal policy. 
We also have an additional error $\varepsilon_{sub}$ that does not go to zero even when we can collect more samples for OPS. Only $\varepsilon_{est}$ goes to zero as $n$ goes to infinity.

To see how poor the guarantee can be due to $\varepsilon_{sub}$, suppose we have two action values $q_1=100 q^*$ and $q_2=q^*+\text{some random noise}$, then $q_1$ has a large Bellman error but $\pi_1$ is actually optimal. 
We would choose $q_2$ since it has a lower Bellman error, even when we can collect an infinite number of samples. 
To obtain a meaningful guarantee, we need to include another value function, for example, $q_3=q^*$ to make $\varepsilon_{sub}=0$. As we collect more samples, we can estimate the Bellman error more accurately and choose $q_3$ eventually.

\subsection{A Sound BE Selection Algorithm with Cross-Validation}
\label{sec:be_efficient}

If the data coverage and suboptimality of the candidate set assumptions hold, the remaining part is to have a sound BE selection algorithm. 
In this subsection, we show that there is a simple BE selection method, which we call Identifiable BE Selection (IBES). 

In deterministic environments, the BE can be easily estimated using TD errors. Given a state-action value $q$ with $v_q(s) \defeq \max_{a} q(s,a)$ the corresponding state value, the BE estimate is 
\begin{equation*}
    \mathrm{TDE}(q) \defeq \frac{1}{|D|}\sum_{(s,a,s',r) \in D} (q(s,a) - r - v_q(s'))^2.
\end{equation*}

In stochastic environments, estimating BE  typically involves an additional regression problem \citep{antos2008learning}.
\citet{antos2008learning} propose an estimator for the BE by introducing an auxiliary function $g\in\cG$ where $\cG$ is a function class: 
\begin{align}
    &\hat \cE(q) 
    = \max_{g \in \cG} [\nonumber\\ 
    & \tfrac{1}{|D|} \!\!\!\sum_{(s,a,s',r) \in D}\!\!\! (q(s,a) - r - v_q(s'))^2 - \tfrac{1}{|D|} \!\!\!\sum_{(s,a,s',r) \in D}\!\!\! (g(s,a) - r - v_q(s'))^2].
    \label{eq:BEE}
\end{align}

We can rewrite the inner term in the equation by the change of variable  \citep{dai2018sbeed,patterson2021generalized}. Let $t(r,s') \defeq r+v_q(s')$ be the target and $\delta(s,a,r,s') \defeq t(r,s')-q(s,a)$ be the TD error to simplify the equations. Consider a new auxiliary function $h(s,a)=g(s,a)-q(s,a)$, then the inner term in Equation \eqref{eq:BEE} is
\begin{align}
    &\tfrac{1}{|D|} \sum_{(s,a,s',r) \in D}\left[ (q(s,a) - r - v_q(s'))^2 - (g(s,a) - r - v_q(s'))^2\right] \nonumber\\
    &= \tfrac{1}{|D|} \sum_{(s,a,s',r) \in D} \left[(t(r,s')-q(s,a))^2 - (t(r,s')-g(s,a))^2\right] \nonumber\\
    &= \tfrac{1}{|D|} \sum_{(s,a,s',r) \in D} \left[2h(s,a)\delta(s,a,r,s') - h(s,a)^2\right].
    \label{eq:IBE}
\end{align}
That is, we can also use an auxiliary function $h$ to predict $\bellman q - q$, instead of $\bellman q$, and use the auxiliary function to estimate the Bellman error.
The benefit is that the Bellman errors are more likely to be predictable. If there is an action-value function $q\in\cQ$ with a small BE, the Bellman errors are nearly zero everywhere and any reasonable function class are able to represent the solution. 

The remaining part is to find a function class that has low approximation error and a low statistical complexity. 
Fortunately, we can perform model selection to choose the function approximation $\cG$.
This is because we are running regression with fixed targets in Eq (\ref{eq:BEE}), and model selection for regression is well-studied. 
For example, 
consider a finite set of potential function classes $\cG_1,\dots,\cG_M$, we can use a holdout validation set or cross-validation to select the best function class and other hyperparameters.
Therefore, we can choose a function class such that it has a low approximation error and a small complexity measure, which can potentially result in improved sample efficiency. 
We describe the full procedure of IBES in Algorithm \ref{alg:be}.

\begin{algorithm}[t]
\caption{Identifiable BE Selection (IBES) with holdout validation}\label{alg:be}
\begin{algorithmic}
\State \textbf{Inputs}: Candidate set $\cQ$, training data $D$, validation data $D_{val}$
\State \textbf{Hyperparameters}: A set of function classes $\cG_1,\dots,\cG_M$ for model selection
\State Let $\delta(s,a,r,s')\defeq r + v_q(s') - q(s,a)$
\For{$q \in \cQ$}
    \For{$m=1,\dots,M$}
        \State Perform regression: 
        \State $\hat g_m \gets \min_{g\in\cG_{m}} \frac{1}{|D|} \sum_D (g(s,a) - \delta(s,a,r,s'))^2$
        \State Compute validation error: 
        \State $l(\hat g_m) \gets \frac{1}{|D_{val}|} \sum_{D_{val}} (\hat g_m(s,a) - \delta(s,a,r,s'))^2$
    \EndFor
    \State Find the best function class: $k \gets \argmin_{m=1,\dots,M} l(\hat g_m)$
    \State Estimate the Bellman error for $q$: 
    \State $\mathrm{BE}(q) \gets \frac{1}{|D|}\sum_D 2 \hat g_k(s,a)\delta(s,a,r,s') - \hat g_k(s,a)^2$
\EndFor
% \State $q^\dagger \gets \arg\min_{q\in\cQ} \mathrm{BE}(q)$
\State Output: $q^\dagger \gets \arg\min_{q\in\cQ} \mathrm{BE}(q)$
\end{algorithmic}
\end{algorithm}

\subsection{Comparison to Existing Methods}
There are other works consider selecting a value function that has the smallest BE or is the closest to the optimal value function. 
\citet{farahmand2011model} consider selecting a value function such that, with high probability, the output value functions has the smallest BE. 
They propose to fit a regression model $\tilde q_i$ to predict $\bellman q_i$ and bound the BE by $\norm{q_i - \tilde q_i}_{2,\mu}^2 + b_i$ where the first term can be viewed as the (empirical) projected Bellman error, and the second term $b_i$ is a high-probability upper bound on the excess risk of the regression model, which is assumed to be given.  
There is a related work on using BE selection method for OPS \citep{zitovsky2023revisiting}.
% As far we know, no prior work has studied this BE selection method for OPS. 
They propose a method called Supervised Bellman Validation (SBV), which is effectively an empirical version of the method from \citet{farahmand2011model}, without an additional upper bound on the excess risk. 
% They did not provide any sample complexity guarantee. 
In our experiments, we find that our method outper  forms SBV in terms of sample efficiency, likely because the auxiliary function $h$ is used to predict the Bellman error, instead of $\bellman q$. 

\citet{zhang2021towards} propose to use a (empirical) projected Bellman error, called BVFT loss, with piecewise constant function classes. Their selection algorithm chooses the value function with the smallest BVFT loss, assuming $q^*$ is in the candidate set (approximately) and a stronger data assumption is satisfied. 
Interestingly, this condition on having $q^*$ is essentially equivalent to our condition requiring small $\varepsilon_{sub}$, since $q^*$ has exactly zero BE. 
The algorithms, though, are quite different from ours.
Their method is computationally expensive since it scales with $O(|\Pi|^2)$ instead of $O(|\Pi|)$, making the method impractical when the candidate set is not small.
Our method requires a weaker data coverage assumption and the computation cost scales linearly with $O(|\Pi|)$. 
Note that the sample complexity of our methods depends on the function class that is used to perform the regression. BVFT can be viewed as using the piecewise constant function classes to fit the Bellman target, the sample complexity depends on the piecewise constant function classes, which is measured by the number of discretization bins $(V_{max}/\varepsilon_{dct})^2$ and $\varepsilon_{dct}$ is the discretization resolution.

\subsection{Comparison between BE and OPE for OPS}
\label{sec:pro_con}

Table \ref{table1} summarizes the comparison of BE and OPE methods.
In theory, OPE methods such as FQE require weaker assumptions compared to IBES and BVFT. One of the strict assumptions for IBES and BVFT is that the candidate policy set contains a nearly optimal action value function, which we often cannot verify in the offline setting. 
This suggests that OPE is a more robust and reliable method for OPS.

On the other hand, if we satisfy these stronger assumptions, then IBES has several advantages. 
IBES can be much more sample efficient in deterministic environments, or even in stochastic environments by choosing appropriate function approximators. 
For IBES, we are not plagued by the issue of having hard-to-specify hyperparameters.
This is critical for the offline setting, where we cannot test different hyperparameter choices in the environment.

\begin{table}[h]
    \caption{A summary of OPS methods. The first two methods are OPE methods and the last two methods are BE-based methods.}
    \centering
    \begin{tabular}{l l l}
      \bfseries Method & \bfseries Data Coverage & \bfseries Suboptimality \\
      \midrule
      IS & Action coverage & No assumption \\
      FQE & Data coverage for $\Pi$ & No assumption \\
      IBES,SBV & Data coverage for $\Pi \cup \{\pi^*\}$ & $\underset{q\in\cQ}{\min} \ \mathcal{E}(q) \leq \varepsilon_{sub}$ \\
      BVFT & Stronger data coverage \citep{zhang2021towards} & $q^* \in \cQ$
    \end{tabular}
    % \vspace{-0.3cm}
    \label{table1}
\end{table}

\section{Experiments}

In this section, we empirically investigate the different BE methods as well as FQE for OPS. The goal in the experiments is to gain a better understanding of how IBES and FQE perform when we vary important factors such as data coverage, sample size, and candidate policies. 
We first compare different BE-based methods to show the advantage of our proposed method IBES, and investigate the differences between IBES and FQE in classic control environments where we vary the data coverage. 
Finally, we perform an experiment on the Atari offline benchmark dataset.

To evaluate the performance of OPS, we consider the normalized top-$k$ regret used in \citet{zhang2021towards}. 
Top-$k$ regret is the gap between the best policy within the top-$k$ policies, and the best policy among all candidate policies.
We then normalize the regret by the difference between the best and the worst policy, so the normalized regret is between $0$ and $1$. 
The normalized regret can be interpreted as the percentage of degradation if we use offline selection instead of online selection, since online selection with Monte Carlo evaluation can always achieve a regret of zero. 
OPS corresponds to $k=1$; for most results we use $k=1$, but include some results for $k > 1$. 
All hyperparameters for OPS methods are selected using only the datasets, not by peaking at performance on the real environment; for more details, see the appendix.

\subsection{Comparison between BE-based Methods}

In our first set of experiments, we compare different BE-based methods for OPS. We conduct experiments on two standard RL environments: Acrobot and Cartpole. 
We also include the stochastic versions of these environments with sticky actions \citep{machado2018revisiting}, which we call Stochastic Acrobot and Stochastic Cartpole.
We generate a set of candidate policy-value pairs by running CQL with different hyperparameters on a batch of data collected by a trained policy with random actions taken $40\%$ of the time, and generate datasets for OPS that provide good data coverage. 
We then use either FQE or IBES for OPS. 
We also included SBV \citep{zitovsky2023revisiting} and BVFT \citep{zhang2021towards},\footnote{We modified the BVFT implementation from the author of \citet{zhang2021towards} (\url{https://github.com/jasonzhang929/BVFT_empirical_experiments/}).} and a random selection baseline which selects a policy uniformly at random from the candidate set.

\begin{figure*}[t]
    \centering
    \includegraphics[width=0.8\textwidth]{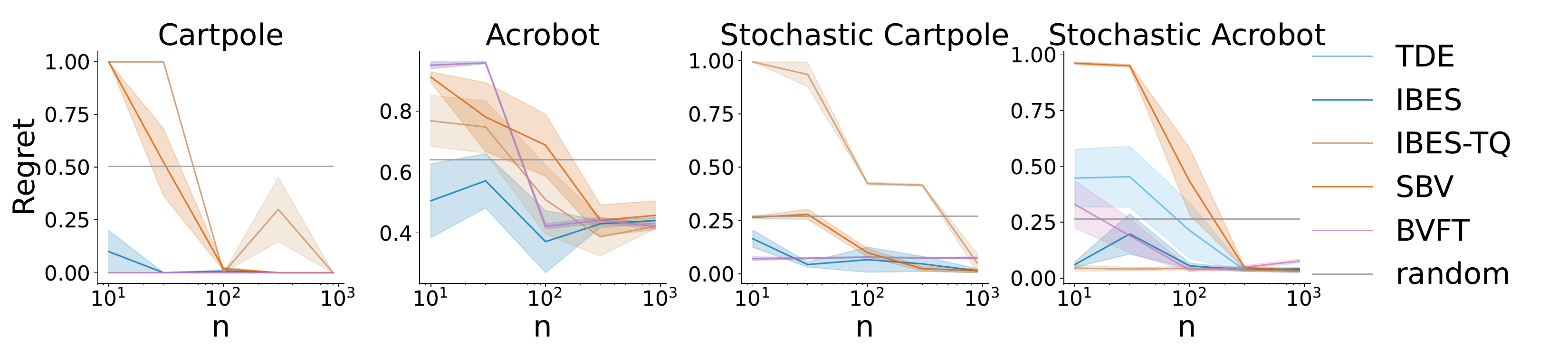}
    \caption{Comparison between BE methods. The figure shows the normalized top-$1$ regret with varying sample size, averaged over 10 runs with one standard error.
    IBES consistently achieves the lowest regret across environments. 
    }
    \Description{AAMAS.}
    \label{fig:exp_target}
\end{figure*}

\begin{figure*}[t]
    \centering
    \includegraphics[width=0.8\textwidth]{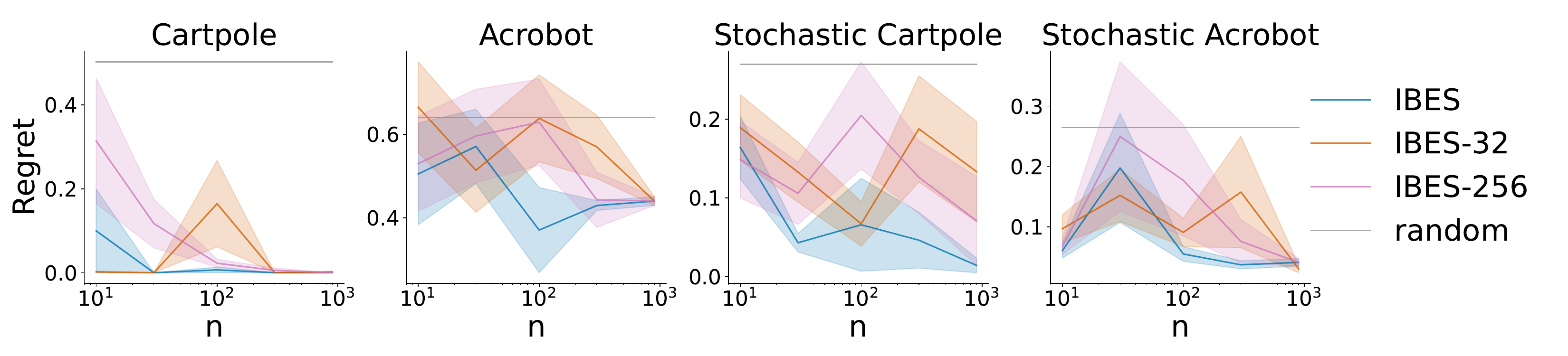}
    \caption{Comparison to BE with a fixed number of hidden units. 
    The figure shows the normalized top-$1$ regret averaged over 10 runs with one standard error.
    IBES with model selection consistently achieves the lowest regret across environments. 
    }
    \Description{AAMAS.}
    \label{fig:be}
\end{figure*}

We first compare different BE methods, and investigate the effect of using the Bellman error $\bellman q-q$ as the target. 
We include a baseline that uses $\bellman q$ as the target, called IBES-TQ. 
Note that we also perform model selection for SBV and IBES-TQ, similar to IBES.
In Figure \ref{fig:exp_target}, we can see all BE methods converge to the same performance as the sample size grows larger, but IBES using the Bellman error as target is much more sample efficient than IBES-TQ and SBV across all environment.
This is likely due to the fact that Bellman error is easier to predict so using a smaller neural network is sufficient and has a better sample efficiency.
We also found BVFT performs similar to TDE (which are overlapping in Cartpole, Acrobot and Stochastic Cartpole), which is likely due to the fact that BVFT with small discretization is equivalent to TDE.

Now we show that the use of validation for model selection in IBES is important.
To do so, we compare IBES (Algorithm 1) which takes as input a set of model classes that varies the number of hidden units to IBES with only a single model class where the network has a fixed number of hidden units.
In our experiment, we use a two layer neural network model as the function approximation, and perform model selection to find the number of hidden units from the set $\{32,64,128,256\}$ for IBES.
In Figure \ref{fig:be}, we include two baselines called IBES-32 and IBES-256, which has the hidden units of size 32 and 256 respectively. 
In general, IBES with model selection achieves a regret less than or equal to that of IBES with a fixed number of hidden units across all environments. 
The performance of IBES with model selection and IBES-256 converges as the sample size gets larger. This result shows an improvement in terms of sample efficiency due to model selection. 
In stochastic environments, IBES with a small number of hidden units does not work well even with a large sample size. 
The results suggest that the ability to perform model selection to balance approximation and estimation error is important to improve sample efficiency while achieving low regret. 

In summary, the experiment suggests the proposed BE method is more sample efficient, due to the fact that the method predicts the Bellman error, and the model selection procedure. 

\subsection{Comparison between FQE and IBES under Different Data Coverage}
In the second set of experiments, we aim to compare BE methods to OPE methods for model selection.
Specifically, we compare IBES to FQE under different degrees of data coverage and different sample sizes.
We design two different datasets for the experiments: (a) well-covered data is generated such that all candidate policies are well-covered, and (b) well-covered data with optimal trajectories includes more diverse trajectories collected by an $\varepsilon$-greedy optimal policy (which used in the previous experiment).

\begin{figure*}[t]
    \centering
    \includegraphics[width=0.73\textwidth]{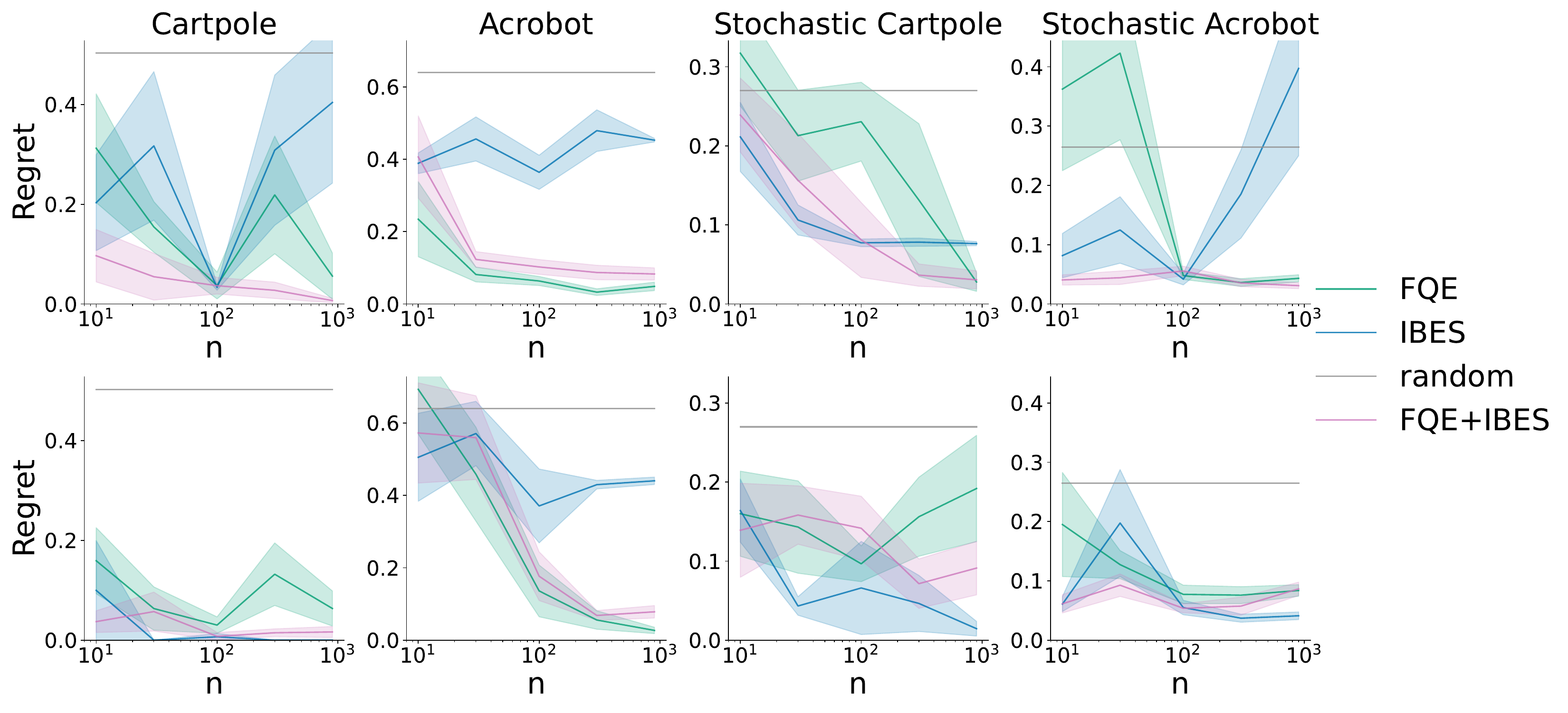}
    \caption{Comparison between IBES and FQE under different sample size and data distributions. 
    The figure shows the normalized top-$1$ regret with varying sample size, averaged over 10 runs with one standard error.
    }
    \Description{AAMAS.}
    \label{fig:exp}
\end{figure*}

Figure \ref{fig:exp} shows the results for top-$1$ regret with varying numbers of episodes where the top row uses dataset (a) and the bottom row uses dataset (b).
We first focus on the asymptotic performance ($n \approx 10^3$) across  different datasets. 
FQE performs very well with a small regret on well-covered and diverse data.
IBES performs better with optimal trajectories, especially in Cartpole. This result matches our theoretical result that IBES requires a stronger data coverage for an optimal policy. 
Moreover, IBES often performs better than FQE with a small sample size, suggesting that it could offer a slightly better sample efficiency than FQE. 

Investigating the results deeper, we observed that when IBES does not perform well (Acrobot with optimal trajectories), it is often the case that one of the value functions has the smallest Bellman error but it is not far from optimal. When FQE does not perform well (Stochastic Cartpole with optimal trajectories), it is often the case that one of the candidate policies is highly overestimated. Therefore, we include a simple two-stage method that first uses FQE to select $k_1$ policies, then use IBES to find the top-$k_2$ policies amongst the $k_1$ selected policies.
We set $k_1=10$ and $k_2=1$ in our experiment. The idea is similar to the two-stage method proposed in \citet{tang2021model}. 
We label the two-stage method as FQE+IBES in the figure. We can see that it performs consistently well, better than either method alone. 
We hypothesize the explanation is that FQE usually performs very well for top-$k$ selection with a large $k$, even though top-$1$ regret might be bad. 
We then use IBES to select from a subset of reasonably good candidate policies, where the candidate with the smallest error is more likely to be optimal. 
Further investigation about combining multiple OPS methods might be a promising research direction. 

\subsection{Comparison between FQE and IBES on Atari Datasets}
Finally, we conduct experiments on benchmark Atari datasets to show the hardness of OPS for offline RL, and compare IBES to FQE in a more practical setting. 
We use the offline data on Breakout and Seaquest from the DQN replay dataset\footnote{\url{https://research.google/resources/datasets/dqn-replay/}}, a commonly used benchmark in offline RL. 
We sample $1$ million transitions from different learning stages to promote data coverage. 
Unlike our previous experiment, the data coverage might be poor due to the absent of explicit exploration to cover all candidate policies.
We use $50\%$ of the data to generate a set of candidate policy-value pairs by running CQL with different number of gradient steps and different regularization coefficients, as specified in \citet{kumar2020conservative}. We use the other $50\%$ data to perform OPS using FQE or IBES. 
We generate two candidate sets for each environment. The early-learning candidate contains policies with gradient steps in $\{50k, \cdots, 500k\}$ and final-learning candidate contains policies with gradient steps in $\{550k, \dots, 1000k\}$.
In this experiment, we use the CQL and FQE implementation from the d3rlpy library \citep{d3rlpy}. 

\begin{figure}[t]
    \centering
    \includegraphics[scale=0.14]{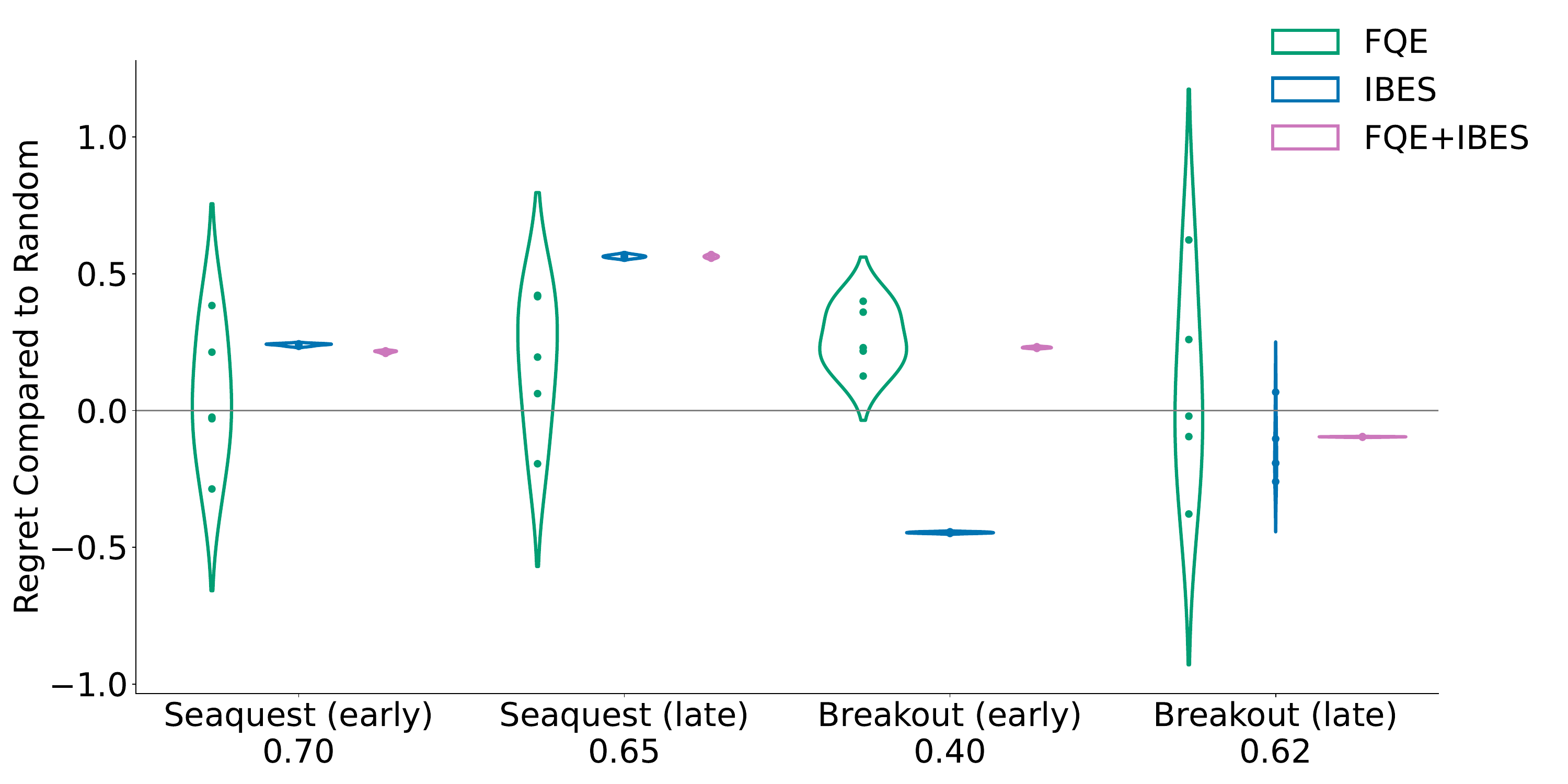}
    \caption{Regret improvement over the random baseline on the Atari dataset. 
    We show the regret compared to the random baseline. A positive value means the method outperforms random selection, and a negative value means the method performs worse than random selection. 
    We show the distribution of the regret improvement across 5 random seeds (each seed is a single point in the violin plots) and the regret of the random baseline under the x-axis labels.  
    None of these methods can consistently outperform the random baseline, showing the hardness of OPS. }
    \Description{AAMAS.}
    \label{fig:regret}
\end{figure}

Figure \ref{fig:regret} shows the top-$1$ regret. 
We can see that these none of the OPS methods can consistently outperform the random baseline, showing the hardness of OPS in a practical situation where we can not control the data coverage. 
For Breakout with early-learning candidate set, we found that IBES tends to pick the candidate value function with a small number of gradient steps. This is likely due to the fact that none of the candidate value functions are close to optimal (that is, large $\varepsilon_{sub}$) and the value function with a small number of training steps has a small magnitude and hence a small estimated Bellman error. 
IBES performs better with the candidate set containing policies in the final learning stage, where it is more likely to contain a value function that is close to optimal. 
This shows the limitation of IBES, as discussed in the previous section. 
We can also see that FQE is more robust to the choice candidate set since it does not require one of the value function being optimal, as long as all candidate set are covered by the data. However, it is more sensitive to the data used to run FQE, and hence high variance across different random seeds. 

\section{Conclusion}
In this paper, we made contributions towards understanding when OPS is feasible for RL.
One of our main results that the sample complexity of OPS is lower-bounded by the sample complexity of OPE implies that without satisfying conditions to make OPE feasible, we cannot do policy selection efficiently. 
Our second contribution is the proposed IBES algorithm, and we empirically show that it is more sample efficient than existing BE-based methods for OPS.   

We expect active research topics will be (1) to identify suitable conditions on the policies, environments and data, to make OPS sample efficient, (2) to design offline RL algorithms that have sound methods to select their own hyperparameters, and (3) to investigate how to combine multiple methods for a better OPS algorithm. 
In offline RL, we cannot select hyperparameters by testing in the real-world, and instead are limited to using the offline data. OPS is arguably one of the most critical steps towards bringing RL into the real-world, and there is much more to understand. 

%%%%%%%%%%%%%%%%%%%%%%%%%%%%%%%%%%%%%%%%%%%%%%%%%%%%%%%%%%%%%%%%%%%%%%%%

%%% The acknowledgments section is defined using the "acks" environment
%%% (rather than an unnumbered section). The use of this environment 
%%% ensures the proper identification of the section in the article 
%%% metadata as well as the consistent spelling of the heading.

\begin{acks}
This research was supported in part by the Natural Sciences and Engineering Research Council of Canada (NSERC), the Canada CIFAR AI Chair Program, and a Canada Research Chair in Reinforcement Learning. Prabhat Nagarajan has been additionally supported by the Alberta Innovates Graduate Student Scholarship. Computational resources were provided in part by the Digital Research Alliance of Canada.
\end{acks}

%%%%%%%%%%%%%%%%%%%%%%%%%%%%%%%%%%%%%%%%%%%%%%%%%%%%%%%%%%%%%%%%%%%%%%%%

%%% The next two lines define, first, the bibliography style to be 
%%% applied, and, second, the bibliography file to be used.

\bibliographystyle{ACM-Reference-Format} 
\bibliography{main}

@inproceedings{zitovsky2023revisiting,
  title={Revisiting {Bellman} Errors for Offline Model Selection},
  author={Zitovsky, Joshua P and De Marchi, Daniel and Agarwal, Rishabh and Kosorok, Michael Rene},
  booktitle={International Conference on Machine Learning},
  year={2023}
}

@article{gulcehre2020rl,
  title={{RL} unplugged: A suite of benchmarks for offline reinforcement learning},
  author={Gulcehre, Caglar and Wang, Ziyu and Novikov, Alexander and Paine, Thomas and G{\'o}mez, Sergio and Zolna, Konrad and Agarwal, Rishabh and Merel, Josh S and Mankowitz, Daniel J and Paduraru, Cosmin and others},
  journal={Advances in Neural Information Processing Systems},
  year={2020}
}

@inproceedings{agarwal2020optimistic,
  title={An optimistic perspective on offline reinforcement learning},
  author={Agarwal, Rishabh and Schuurmans, Dale and Norouzi, Mohammad},
  booktitle={International Conference on Machine Learning},
  year={2020},
}

@article{machado2018revisiting,
  title={Revisiting the arcade learning environment: Evaluation protocols and open problems for general agents},
  author={Machado, Marlos C.\ and Bellemare, Marc G.\ and Talvitie, Erik and Veness, Joel and Hausknecht, Matthew and Bowling, Michael},
  journal={Journal of Artificial Intelligence Research},
  year={2018}
}

@inproceedings{xie2020q,
  title={Q* approximation schemes for batch reinforcement learning: A theoretical comparison},
  author={Xie, Tengyang and Jiang, Nan},
  booktitle={Conference on Uncertainty in Artificial Intelligence},
  year={2020}
}

@article{bottou2013counterfactual,
  title={Counterfactual Reasoning and Learning Systems: The Example of Computational Advertising.},
  author={Bottou, L{\'e}on and Peters, Jonas and Qui{\~n}onero-Candela, Joaquin and Charles, Denis X and Chickering, D Max and Portugaly, Elon and Ray, Dipankar and Simard, Patrice and Snelson, Ed},
  journal={Journal of Machine Learning Research},
  year={2013}
}

@article{lee2022oracle,
  title={Oracle Inequalities for Model Selection in Offline Reinforcement Learning},
  author={Lee, Jonathan N and Tucker, George and Nachum, Ofir and Dai, Bo and Brunskill, Emma},
  journal={Advances in Neural Information Processing Systems},
  year={2022}
}

@article{paine2020hyperparameter,
  title={Hyperparameter selection for offline reinforcement learning},
  author={Paine, Tom Le and Paduraru, Cosmin and Michi, Andrea and Gulcehre, Caglar and Zolna, Konrad and Novikov, Alexander and Wang, Ziyu and de Freitas, Nando},
  journal={arXiv preprint arXiv:2007.09055},
  year={2020}
}

@article{patterson2021generalized,
  title={A generalized projected {Bellman} error for off-policy value estimation in reinforcement learning},
  author={Patterson, Andrew and White, Adam and White, Martha},
  journal={Journal of Machine Learning Research},
  year={2022}
}

@inproceedings{dai2018sbeed,
  title={{SBEED}: Convergent reinforcement learning with nonlinear function approximation},
  author={Dai, Bo and Shaw, Albert and Li, Lihong and Xiao, Lin and He, Niao and Liu, Zhen and Chen, Jianshu and Song, Le},
  booktitle={International Conference on Machine Learning},
  year={2018}
}

@article{levine2020offline,
  title={Offline reinforcement learning: Tutorial, review, and perspectives on open problems},
  author={Levine, Sergey and Kumar, Aviral and Tucker, George and Fu, Justin},
  journal={arXiv preprint arXiv:2005.01643},
  year={2020}
}

@article{nie2022data,
  title={Data-Efficient Pipeline for Offline Reinforcement Learning with Limited Data},
  author={Nie, Allen and Flet-Berliac, Yannis and Jordan, Deon R and Steenbergen, William and Brunskill, Emma},
  journal={Advances in Neural Information Processing Systems},
  year={2022}
}

@inproceedings{tang2021model,
  title={Model selection for offline reinforcement learning: Practical considerations for healthcare settings},
  author={Tang, Shengpu and Wiens, Jenna},
  booktitle={Machine Learning for Healthcare Conference},
  year={2021}
}

@article{farahmand2011model,
  title={Model selection in reinforcement learning},
  author={Farahmand, Amir-massoud and Szepesv{\'a}ri, Csaba},
  journal={Machine Learning},
  year={2011},
  publisher={Springer}
}

@article{zhang2021towards,
  title={Towards hyperparameter-free policy selection for offline reinforcement learning},
  author={Zhang, Siyuan and Jiang, Nan},
  journal={Advances in Neural Information Processing Systems},
  year={2021}
}

@inproceedings{wang2020statistical,
  title={What are the statistical limits of offline {RL} with linear function approximation?},
  author={Wang, Ruosong and Foster, Dean P and Kakade, Sham M},
  booktitle={International Conference on Learning Representations},
  year={2021}
}

@inproceedings{lee2022model,
  title={Model selection in batch policy optimization},
  author={Lee, Jonathan and Tucker, George and Nachum, Ofir and Dai, Bo},
  booktitle={International Conference on Machine Learning},
  year={2022}
}

@inproceedings{duan2021risk,
  title={Risk bounds and {Rademacher} complexity in batch reinforcement learning},
  author={Duan, Yaqi and Jin, Chi and Li, Zhiyuan},
  booktitle={International Conference on Machine Learning},
  year={2021}
}

@inproceedings{xiao2022curse,
  title={The Curse of Passive Data Collection in Batch Reinforcement Learning},
  author={Xiao, Chenjun and Lee, Ilbin and Dai, Bo and Schuurmans, Dale and Szepesvari, Csaba},
  booktitle={International Conference on Artificial Intelligence and Statistics},
  year={2022}
}

@article{bartlett2002model,
  title={Model selection and error estimation},
  author={Bartlett, Peter L and Boucheron, St{\'e}phane and Lugosi, G{\'a}bor},
  journal={Machine Learning},
  year={2002}
}

@article{yin2021optimal,
  title={Optimal uniform {OPE} and model-based offline reinforcement learning in time-homogeneous, reward-free and task-agnostic settings},
  author={Yin, Ming and Wang, Yu-Xiang},
  journal={Advances in Neural Information Processing Systems},
  year={2021}
}

@article{wu2019behavior,
  title={Behavior regularized offline reinforcement learning},
  author={Wu, Yifan and Tucker, George and Nachum, Ofir},
  journal={arXiv preprint arXiv:1911.11361},
  year={2019}
}

@inproceedings{su2020adaptive,
  title={Adaptive estimator selection for off-policy evaluation},
  author={Su, Yi and Srinath, Pavithra and Krishnamurthy, Akshay},
  booktitle={International Conference on Machine Learning},
  year={2020}
}

@inproceedings{hallak2013model,
  title={Model selection in markovian processes},
  author={Hallak, Assaf and Di-Castro, Dotan and Mannor, Shie},
  booktitle={International Conference on Knowledge Discovery and Data Mining},
  year={2013}
}

@inproceedings{le2019batch,
  title={Batch policy learning under constraints},
  author={Le, Hoang and Voloshin, Cameron and Yue, Yisong},
  booktitle={International Conference on Machine Learning},
  year={2019}
}

@book{sutton2018reinforcement,
  title={Reinforcement learning: An introduction},
  author={Sutton, Richard S and Barto, Andrew G},
  year={2018},
  publisher={MIT press}
}

@article{fu2020d4rl,
  title={{D4RL}: Datasets for deep data-driven reinforcement learning},
  author={Fu, Justin and Kumar, Aviral and Nachum, Ofir and Tucker, George and Levine, Sergey},
  journal={arXiv preprint arXiv:2004.07219},
  year={2020}
}

@inproceedings{yang2022offline,
  title={Offline policy selection under uncertainty},
  author={Yang, Mengjiao and Dai, Bo and Nachum, Ofir and Tucker, George and Schuurmans, Dale},
  booktitle={International Conference on Artificial Intelligence and Statistics},
  year={2022}
}

@article{antos2008learning,
  title={Learning near-optimal policies with {Bellman}-residual minimization based fitted policy iteration and a single sample path},
  author={Antos, Andr{\'a}s and Szepesv{\'a}ri, Csaba and Munos, R{\'e}mi},
  journal={Machine Learning},
  year={2008},
  publisher={Springer}
}

@article{krishnamurthy2016pac,
  title={{PAC} reinforcement learning with rich observations},
  author={Krishnamurthy, Akshay and Agarwal, Alekh and Langford, John},
  journal={Advances in Neural Information Processing Systems},
  year={2016}
}

@article{kumar2020conservative,
  title={Conservative {Q}-learning for offline reinforcement learning},
  author={Kumar, Aviral and Zhou, Aurick and Tucker, George and Levine, Sergey},
  journal={Advances in Neural Information Processing Systems},
  year={2020}
}

@article{foster2019model,
  title={Model selection for contextual bandits},
  author={Foster, Dylan J and Krishnamurthy, Akshay and Luo, Haipeng},
  journal={Advances in Neural Information Processing Systems},
  year={2019}
}

@inproceedings{lee2021online,
  title={Online model selection for reinforcement learning with function approximation},
  author={Lee, Jonathan and Pacchiano, Aldo and Muthukumar, Vidya and Kong, Weihao and Brunskill, Emma},
  booktitle={International Conference on Artificial Intelligence and Statistics},
  year={2021}
}

@inproceedings{thomas2016data,
  title={Data-efficient off-policy policy evaluation for reinforcement learning},
  author={Thomas, Philip and Brunskill, Emma},
  booktitle={International Conference on Machine Learning},
  year={2016}
}

@book{lattimore2020bandit,
  title={Bandit algorithms},
  author={Lattimore, Tor and Szepesv{\'a}ri, Csaba},
  year={2020},
  publisher={Cambridge University Press}
}

@inproceedings{kumar2021workflow,
  title={A workflow for offline model-free robotic reinforcement learning},
  author={Kumar, Aviral and Singh, Anikait and Tian, Stephen and Finn, Chelsea and Levine, Sergey},
  booktitle={Conference on Robot Learning},
  year={2021}
}

@article{d3rlpy,
  author  = {Takuma Seno and Michita Imai},
  title   = {d3rlpy: An Offline Deep Reinforcement Learning Library},
  journal = {Journal of Machine Learning Research},
  year    = {2022}
}

@article{yu2021combo,
  title={Combo: Conservative offline model-based policy optimization},
  author={Yu, Tianhe and Kumar, Aviral and Rafailov, Rafael and Rajeswaran, Aravind and Levine, Sergey and Finn, Chelsea},
  journal={Advances in Neural Information Processing Systems},
  year={2021}
}

@inproceedings{trabucco2021conservative,
  title={Conservative objective models for effective offline model-based optimization},
  author={Trabucco, Brandon and Kumar, Aviral and Geng, Xinyang and Levine, Sergey},
  booktitle={International Conference on Machine Learning},
  year={2021},
}

@article{qi2022data,
  title={Data-driven offline decision-making via invariant representation learning},
  author={Qi, Han and Su, Yi and Kumar, Aviral and Levine, Sergey},
  journal={Advances in Neural Information Processing Systems},
  year={2022}
}

@inproceedings{cheng2022adversarially,
  title={Adversarially trained actor critic for offline reinforcement learning},
  author={Cheng, Ching-An and Xie, Tengyang and Jiang, Nan and Agarwal, Alekh},
  booktitle={International Conference on Machine Learning},
  year={2022}
}

@article{fujimoto2021minimalist,
  title={A minimalist approach to offline reinforcement learning},
  author={Fujimoto, Scott and Gu, Shixiang Shane},
  journal={Advances in Neural Information Processing Systems},
  year={2021}
}

@inproceedings{kostrikov2021offline,
  title={Offline reinforcement learning with implicit {Q}-learning},
  author={Kostrikov, Ilya and Nair, Ashvin and Levine, Sergey},
  booktitle={International Conference on Learning Representations},
  year={2021}
}

@article{kumar2019stabilizing,
  title={Stabilizing off-policy q-learning via bootstrapping error reduction},
  author={Kumar, Aviral and Fu, Justin and Soh, Matthew and Tucker, George and Levine, Sergey},
  journal={Advances in Neural Information Processing Systems},
  year={2019}
}

@article{mnih2013playing,
  title={Playing atari with deep reinforcement learning},
  author={Mnih, Volodymyr and Kavukcuoglu, Koray and Silver, David and Graves, Alex and Antonoglou, Ioannis and Wierstra, Daan and Riedmiller, Martin},
  journal={arXiv preprint arXiv:1312.5602},
  year={2013}
}

%%%%%%%%%%%%%%%%%%%%%%%%%%%%%%%%%%%%%%%%%%%%%%%%%%%%%%%%%%%%%%%%%%%%%%%%

\newpage
\onecolumn
\appendix

\section{Related Work}
In this section we provide a more comprehensive survey of prior work on model selection for RL. 
In the online setting, model selection has been studied extensively across contextual bandits \citep{foster2019model} to RL \citep{lee2021online}.
In the online setting, the goal is to select model classes while balancing exploration and exploitation to achieve low regret, which is very different from the offline setting where no exploration is performed. 

In the offline setting, aside from OPE and BE selection, other work on model selection in RL is in other settings: selecting amongst models and selecting amongst OPE estimators.
\citet{hallak2013model} consider model selection for model-based RL algorithms with batch data. 
They focus on selecting the most suitable model that generates the observed data, based on the maximum likelihood framework.
\citet{su2020adaptive} consider adaptive estimation for OPE when the candidate estimators can be ordered with monotonically increasing biases and decreasing confidence intervals. 

In offline RL pipelines, we often split the offline data into training data for generating multiple candidate policies and validation data for selecting the best candidate policy.
\citet{nie2022data} highlight the utility of performing multiple random data splits for OPS.
They do not study the hardness or sample complexity of this procedure.

To the best of our knowledge, there is no previous work on understanding the fundamental limits for OPS in RL. There is one related work in the batch contextual bandit setting, studying the selection of a linear model \citep{lee2022model}.
They provide a hardness result suggesting it is impossible to achieve an oracle inequality that balances the approximation error, the complexity of the function class, and data coverage. 
Our work considers the more general problem of selecting a policy from a set of candidate policies in the RL setting.

\section{Technical Details} 
\label{sec:proof}

\subsection{Proof of Theorem \ref{thm:upper_bound}} \label{sec:proof_upper_bound}

\opsupperbound*
\begin{proof}
    The OPS algorithm $\cL(D,\Pi)$ for a given $(\varepsilon,\delta)$ works as follows: we query an $(\varepsilon',\delta')$-sound OPE algorithm for each policy in $\Pi$ and select the policy with the highest estimated value. That is, $\cL(D, \Pi)$ outputs the policy $\bar{\pi} \defeq \argmax_{\pi \in \Pi} \hat{J} (\pi)$, where $\hat{J}(\pi)$ is the value estimate for policy $\pi$ by the $(\varepsilon', \delta')$-sound OPE algorithm with data $D$.
    
    By definition of an $(\varepsilon', \delta')$-sound OPE algorithm we have 
    \begin{align*}
        &\Pr_{D \sim d_b} (|\hat{J}(\pi) - J(\pi)| \le \varepsilon') \geq 1 - \delta',\forall \pi \in \Pi.
    \end{align*}
    Applying the union bound, we have
    \begin{align*}
        \Pr_{D \sim d_b} (\forall{\pi \in \Pi}, |\hat{J}(\pi) - J(\pi)| \le \varepsilon') & \ge 1 - \delta' |\Pi|.
    \end{align*}
    
    Let $\pi^{\dagger}$ denote the best policy in the candidate set $\Pi$, that is, $\pi^{\dagger} \defeq \argmax_{\pi \in \Pi} J(\pi)$.
    With probability $1 - \delta' |\Pi|$, we have 
    \begin{align*}
        J(\bar{\pi}) & \ge \hat{J}(\bar{\pi})- \epsilon'
        \ge \hat{J}(\pi^{\dagger}) - \varepsilon'
        \ge J(\pi^{\dagger}) - 2\varepsilon'.
    \end{align*}
    The second inequality follows from the definition of $\bar{\pi}$.
    Finally, by setting $\delta' = \delta/|\Pi|$ and $\varepsilon'=\varepsilon/2$, we get
    \begin{align*}
        \Pr_{D \sim d_b} (J(\bar{\pi}) \ge J(\pi^{\dagger}) - \varepsilon) \ge 1 - \delta.
    \end{align*}
    That is, $\cL$ is an $(\varepsilon,\delta)$-sound OPS algorithm, and it requires $O(\mathrm{N}_{OPE}(S,A,H,2/\varepsilon,|\Pi|/\delta))$ episodes. 
\end{proof}

\subsection{Proof of Theorem \ref{thm:lower_bound}} \label{sec:proof_lower_bound}

\opslowerbound*

\begin{proof}
    Our goal is to construct an $(\varepsilon,\delta)$-sound OPE algorithm with $\delta \in (0,1)$ and $\varepsilon \in [0, V_{max}/2]$. 
    To evaluate any policy $\pi$ in $M$ with dataset $D$ sampled from $d_b$, we first construct a new MDP $M_r$ with two additional states: an initial state $s_0$ and a terminal state $s_1$. 
    Taking $a_1$ at $s_0$ transitions to $s_1$ with reward $r$. Taking $a_2$ at $s_0$ transitions to the initial state in the \textit{original} MDP $M$. See Figure \ref{fig:proof2} in the main paper for visualization. 
    
    Let $\Pi=\{\pi_1,\pi_2\}$ be the candidate set in $M_r$ where $\pi_1(s_0) = a_1$ and $\pi_2(s_0) = a_2$ and $\pi_2(a|s)=\pi(a|s)$ for all $(s,a) \in \cS\times\cA$. 
    Since $\pi_1$ always transitions to $s_1$, it never transitions to states in MDP $M$. 
    Therefore, $\pi_1$ can be arbitrary for  all $(s,a) \in \cS\times\cA$. 
    We can add any number of transitions $(s_0,a_1,r,s)$ and $(s_0,a_2,0,s)$ in $D$ to construct the dataset $D_r$ with distribution $d_{b,r}$ arbitrarily.
    
    Suppose we have an $(\varepsilon',\delta')$-sound OPS algorithm, where we set 
    $\varepsilon' = 2\varepsilon/3$ and
    $\delta' = \delta / m$ with $m\defeq\lceil\log(V_{max}/\varepsilon')\rceil$. 
    Note that if this assumption does not hold, then it directly implies that the sample complexity of OPS is larger than $\Omega(\mathrm{N}_{OPE}(S,A,H,1/\varepsilon,1/\delta))$.
    Our strategy will be to iteratively set the reward $r$ of MDP $M_r$ and run our sound OPS algorithm on $\Pi$ and using bisection search to estimate a precise interval for $J(\pi)$. 

    The process is as follows. 
    By construction, our OPS algorithm will output either $\pi_1$, which has value $J_{M_r}(\pi_1) = r$, or output $\pi_2$, which has value $J_{M_r}(\pi_2)=J_{M}(\pi)$. 
    That is, it has the same value as $\pi$ in the original MDP. 
    Let us consider the following two cases. 
    Let $\pi^{\dagger}$ be the best policy in $\Pi$ for MDP $M_r$.
    
    \textbf{Case 1: the OPS algorithm selects $\pi_1$.}
    We know, by definition of a sound OPS algorithm, that
    \begin{align*}
        & \Pr (J_{M_r}(\pi_1) \ge J_{M_r}(\pi^{\dagger}) - \varepsilon') \ge 1 - \delta' \\
        \implies & \Pr (r \ge \text{max}(r, J_{M_r}(\pi_2)) - \varepsilon') \ge 1 - \delta' \\
        \implies & \Pr (J_{M_r}(\pi_2) \leq r + \varepsilon') \ge 1 - \delta'.
    \end{align*}
    
    \textbf{Case 2: the OPS algorithm selects $\pi_2$.}{
       \begin{align*}
        & \Pr (J_{M_r}(\pi_2) \ge J_{M_r}(\pi^{\dagger}) - \varepsilon') \ge 1 - \delta' \\
        \implies & \Pr (J_{M_r}(\pi_2) \ge \text{max}(r, J_{M_r}(\pi_2)) - \varepsilon') \ge 1 - \delta' \\
        \implies & \Pr (J_{M_r}(\pi_2) \ge r - \varepsilon') \ge 1 - \delta'.
    \end{align*}
    }
    
    Given this information, we describe the iterative process by which we produce the estimate $\hat{J}(\pi)$. 
    We first set $U=V_{max},L=0$ and $r = \frac{U+L}{2}$ and run the sound OPS algorithm with input $D_r$ of sample size $n_r$ and the candidate set $\Pi$. 
    Then if the selected policy is $\pi_1$, then we conclude the desired event $J(\pi) \leq r + \varepsilon'$ occurs with probability at least $1-\delta'$, and set $U$ equal to $r$. 
    If the selected policy is $\pi_2$, then we know the desired event $J(\pi) \geq r - \varepsilon'$ occurs with probability at least $1-\delta'$, and set $L$ equal to $r$. 
    We can continue the bisection search until the accuracy is less than $\varepsilon'$, that is, $U-L\leq\varepsilon'$, and the output value estimate is $\hat J(\pi) = \frac{U+L}{2}$.
    
    If all desired events at each call occur, then we conclude that 
    $L-\varepsilon' \leq J(\pi) \leq U + \varepsilon'$
    and thus
    $|J(\pi) - \hat J(\pi)| \leq \varepsilon$. 
    The total number of OPS calls is at most $m$.
    Setting $\delta'=\delta/m$ and applying a union bound, we can conclude that with probability at least $1-\delta$, 
    $|J(\pi) - \hat J(\pi)| \leq \varepsilon$.

    Finally, since any $(\varepsilon,\delta)$-sound OPE algorithm on the instance $(M,d_b,\pi)$ needs at least $\Omega(\mathrm{N}_{OPE}(S,A,H,\varepsilon,1/\delta))$ samples, the $(\varepsilon',\delta')$-sound OPS algorithm also needs at least $\Omega(N_{OPE}(S,A,H,\varepsilon,1/\delta))$, or equivalently $\Omega(\mathrm{N}_{OPE}(S,A,H,3/(2\varepsilon'),1/(m\delta')))$ samples for at least one of the instances $(M_r,d_{b,r},\Pi)$.
\end{proof}

\subsection{Proof of Corollary \ref{corollary:ope_lower_bound}} \label{sec:proof_IS}

\explowerbound*
\begin{proof}
    We provide a proof which uses the construction from \citet{xiao2022curse}. They provide the result for the offline RL problem with Gaussian rewards. Here we provide the result for the OPE problem with Bernoulli rewards since we assume rewards are bounded.

    \begin{figure}
        \centering
        \includegraphics[width=0.6\linewidth]{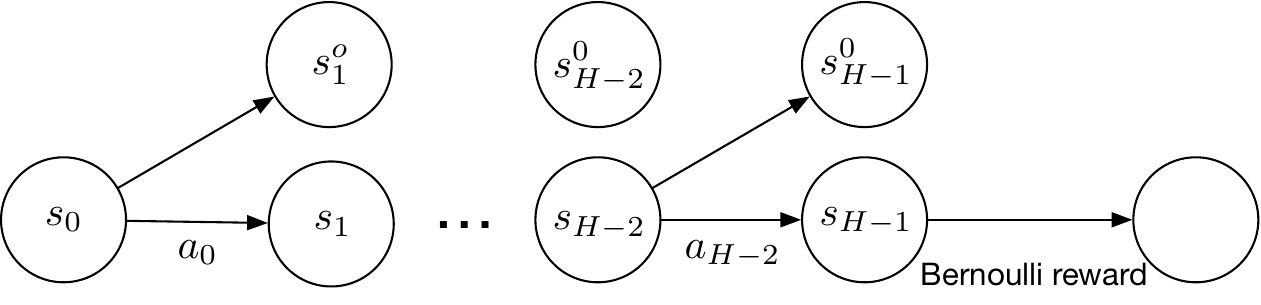}
        \caption{Lower bound construction.}
        \Description{AAMAS.}
        \label{fig:proof5}
    \end{figure}

    We can construct an MDP with $S$ states, $A$ actions and $2H$ states. See Figure \ref{fig:proof5} for the construction. Given any behavior policy $\pi_b$, let $a_h=\argmin_a \pi_b(a|s_h)$ be the action that leads to the next state $s_{h+1}$ from state $s_h$, and all other actions lead to an absorbing state $s_h^o$. 
    Once we reach an absorbing state, the agent gets zero reward for all actions for the remainder of the episode. 
    The only nonzero reward is in the last state $s_{H-1}$.
    Consider a target policy that chooses $a_h$ for state $s_h$ for all $h=0,\dots,H-1$, and two MDPs where the only difference between them is the reward distribution in $s_{H-1}$. MDP 1 has Bernoulli distribution with mean $1/2$ and MDP 2 has Bernoulli distribution with mean $1/2-2\varepsilon$. 
    Let $\bbP_1$ denote the probability
    measure with respect to MDP 1 and $\bbP_2$ denote the probability
    measure with respect to MDP 2.  
    
    Let $\hat r$ denote the OPE estimate by an algorithm $\cL$. 
    Define an event $E=\{\hat r < \frac{1}{2}-\varepsilon\}$. 
    Then $\cL$ is not $(\varepsilon,\delta)$-sound if $(\bbP_1(E) + \bbP_2(E^c))/2 \geq \delta$.
    This is because $\cL$ is not $(\varepsilon,\delta)$-sound 
    if either $\bbP_1(\hat r < \frac{1}{2}-\varepsilon) \geq\delta$ or $\bbP_2(\hat r > \frac{1}{2}-\varepsilon) \geq\delta$.
    
    Using the Bretagnolle–Huber inequality (see Theorem 14.2 of \citet{lattimore2020bandit}),
    we know 
    \begin{align*}
        \frac{\bbP_1(E) + \bbP_2(E^c)}{2} \geq \frac{1}{4}\exp{(-D_{KL}(\bbP_1,\bbP_2))}.
    \end{align*}
    By the chain rule for  KL-divergence and the fact that $\bbP_1$ and $\bbP_2$ only differ in the reward for $(s_{H-1},a_{H-1})$,
    we have 
    \begin{align*}
        D_{KL}(\bbP_1,\bbP_2) 
        &=  \EE_1\left[\sum_{i=1}^n \II\{S_{H-1}^{(i)}=s_{H-1},A_{H-1}^{(i)}=a_{H-1}\}\left(\frac{1}{2}\log{(\frac{1/2}{1/2-\varepsilon})} + \frac{1}{2}\log{(\frac{1/2}{1/2+\varepsilon})}\right)\right] \\
        &= \sum_{i=1}^n \bbP_1(S_{H-1}^{(i)}=s_{H-1},A_{H-1}^{(i)}=a_{H-1}) \left(-\frac{1}{2}\log{(1-4\varepsilon^2)}\right) \\
        &\leq \frac{n8\varepsilon^2}{A^H}.
    \end{align*}
    The last inequality follows from $-\log{(1-4\varepsilon^2)} \leq 8\varepsilon^2$ if $4\varepsilon^2 \leq 1/2$ \citep{krishnamurthy2016pac} and $\bbP_1(S_{H-1}^{(i)}=s_{H-1},A_{H-1}^{(i)}=a_{H-1})<1/A^H$ from the construction of the MDPs.
    
    Finally,
    \begin{align*}
        \frac{\bbP_1(E) + \bbP_2(E^c)}{2} \geq \frac{1}{4}\exp{(-\frac{n8\varepsilon^2}{A^H})}
    \end{align*}
    which is larger than $\delta$ if $n\leq A^H\ln(1/4\delta)/8\varepsilon^2$. 
    As a result, we need at least $\Omega(A^H \ln{(1/\delta)}/\varepsilon^2)$ episodes.
\end{proof}

\subsection{Error amplification for OPS using BE selection} \label{sec:proof_BE}

We first present the telescoping performance difference lemma, originally from Theorem 2 of \citet{xie2020q} for discounting setting, and Lemma 3.2 of \citet{duan2021risk} for the finite horizon setting.
In the following results, we will write $q_h(s,a)$ with the horizon $h$ explicitly for clarity. 

\begin{mylemma}{A.2}[Lemma 3.2 of \citet{duan2021risk}]
    Assume there exists a constant $C$ such that, for any $\pi\in\Pi \cup \{\pi^*\}$,  $\max_{h} \max_{s\in\cS_h,a\in\cA_h} \frac{d_h^\pi(s,a)}{\mu_h(s,a)} \leq C$. 
    For any $q\in\cQ$, let $\pi$ denotes the greedy policy with respect to $q=(q_0,\dots,q_{H-1})$, then
    \begin{align*}
        J(\pi) \geq  J(\pi^*) - 2 H \sqrt{C} \sqrt{\frac{1}{H}\sum_{h=0}^{H-1} \norm{q_h - \bellman q_{h+1}}_{2,\mu_h}} \ .
    \end{align*}
    \label{lemman:be_to_suboptimality}
\end{mylemma}

\begin{restatable}[Error amplification for OPS using BE selection]{corollary}{beerror}
    \label{lemma:be}
    Suppose 
    \begin{enumerate}
        \item there exists a constant $C$ such that $\forall\pi\in\Pi \cup \{\pi^*\}$,  $\underset{h\in [H]}{\max}  \underset{s\in\cS_h,a\in\cA_h}{\max} \frac{d_h^\pi(s,a)}{\mu_h(s,a)} \leq C$,
        \item the suboptimality of the candidate set is small, that is, $\underset{q\in\cQ}{\min} \ \mathcal{E}(q) \leq \varepsilon_{sub}$,
        \item there exists an $(\varepsilon_{est},\delta)$-sound BE selection algorithm $\cL$ on $\cQ$,
    \end{enumerate}
    % Suppose an $(\varepsilon_{est},\delta)$-sound BE selection algorithm outputs $q\in\cQ$, 
    then the OPS algorithm outputs the greedy policy with respect to $\cL(D,\cQ)$ is $(2 H \sqrt{C(\varepsilon_{sub}+\varepsilon_{est})},\delta)$-sound.
\end{restatable}
\begin{proof}
    Since $\mathcal{E}(q_k) = \frac{1}{H} \sum_h \norm{q_{k,h} - \bellman q_{k,h+1}}_{\mu}^2 \leq \varepsilon_{sub} + \varepsilon_{est}$ with probability at least  $1-\delta$, Lemma \ref{lemman:be_to_suboptimality} implies that
     \begin{align*}
        J(\pi_k) 
        &\geq  J(\pi^*) - 2 H \sqrt{C(\varepsilon_{sub}+\varepsilon_{est})} 
        \geq  J(\pi^\dagger) - 2 H \sqrt{C(\varepsilon_{sub}+\varepsilon_{est})}
    \end{align*}
    where $\pi^\dagger$ is the best performing policy in $\Pi$. 
    
    By definition, the OPS algorithm is $(2 H \sqrt{C(\varepsilon_{sub}+\varepsilon_{est})},\delta)$-sound for this instance. 
\end{proof}

\section{Experimental Details} \label{sec:experiment_detail}

In this section, we provide the experimental details for classic RL environments and Atari environments.

\subsection{FQE Implementation}

We first describe FQE in Algorithm \ref{alg:fqe}. Note that it is unclear how to perform model selection for FQE, as a result, we fix the function approximator as a two layer neural network with hidden size $256$ for classic control experiments, and as a convolutional neural network followed by single linear layer for Atari experiments, using the architecture of \citet{mnih2013playing}.

\begin{algorithm}[t]
\caption{OPS using FQE}\label{alg:fqe}
\begin{algorithmic}
\State \textbf{Input}: Candidate set $\Pi$, evaluation data $D$
\State \textbf{Hyperparameters}: Function class $\cF$
\For{$\pi\in\Pi$}
\State Define $q_H(s,a)=0$ for all $(s,a)$
\For {$h=H-1,\dots,0$}
\State $q_h \gets \arg\min_{f\in\cF} \hat l_h(f, q_{h+1})$ where
\begin{equation*}
    \hat l_h(f, q_{h+1})
    \defeq \frac{1}{|D_h|} \sum_{(s,a,r,s')\in D_h} (f(s,a) - r -  q_{h+1}(s',\pi(s')))^2
\end{equation*}
\EndFor
\State Estimate the policy value $\hat J(\pi) 
    \gets \EE_{a\sim\pi(\cdot|s_0)}[ q_0(s_0,a) ]$
\EndFor
\State Output: $\pi^\dagger \gets \arg\max_{\pi\in\Pi} \hat J(\pi)$
\end{algorithmic}
\end{algorithm}

It is known that FQE can diverge, due to the fact that it combines off-policy learning with bootstrapping and function approximation, known as the deadly triad \citep{sutton2018reinforcement}. 
If one of the candidate policies is not well-covered, then the FQE estimate may overestimate the value of the uncovered policy (or even diverge to a very large value) and resulting in poor OPS. 
To circumvent the issue of uncovered policies, we need assign low value estimates for uncovered policies. 
In our FQE implementation, we assign low value estimates to policies for which FQE diverges so the OPS algorithm would not choose these policies. 
That is, if the value estimate is above a threshold $U=V_{max}+100$, we would not choose the policy.

\subsection{Classic RL Experiments}
\paragraph{Stochastic environments.}
We implement stochastic environments by sticky actions. That is, when the agent selects an action to the environment, the action might be repeated with probability $25\%$, up to a maximum of $4$ repeats. 

\paragraph{Generating candidate policies.}
To generate a set of candidate policies, we run CQL with different hyperparameter configurations on a batch of data with $300$ episodes collected with an $\varepsilon$-greedy policy with respect to a trained policy where $\varepsilon=0.4$.
The hyperparameter configurations include:
\begin{itemize}
    \item Learning rate $\in\{0.001, 0.0003, 0.0001\}$
    \item Network hidden layer size $\in\{128,256,512\}$
    \item Regularization coefficient $\in\{1.0,0.1,0.01,0.001,0.0\}$
    \item Iterations of CQL $\in\{100,200\}$
\end{itemize}
As a result, we have $90$ candidate policies. 

\paragraph{Generating data for OPS}
To generate data for offline policy selection, we use two different data distributions. (a) a data distribution collected by running a mixture of all candidate policies. As a result, the data distribution covers all candidate policies well; and (b) a data distribution collected by running the mixture of all candidate policies and an $\varepsilon$-greedy optimal policy that provides more diverse trajectories than (a).

\paragraph{Randomness across multiple runs.}
To perform experiments with multiple runs, we fix the offline data and the candidate policies and only resample the offline data for OPS. This better reflects the theoretical result that the randomness is from resampling the data for an OPS algorithm. 
In our experiments, we use $10$ runs and report the average regret with one standard error.
Since the variability across runs is not large, we find using $10$ runs is enough.

\paragraph{Random selection baseline.}
We include a random selection baseline that randomly chooses $k$ policies. Since the random selection algorithm has very high variance, we estimate the expected regret of random selection by repeating the random selection $10000$ times, and report the average regret.

\paragraph{BVFT}
BVFT has a hyperparameter: the discretization resolution. We follow the method described in the original paper to search for the best resolution from a set of predefined values. Note that in the authors' implementation, they use different sets for different environments. 

\subsection{Atari Experiments}

\paragraph{Generating candidate policies.}
To generate a set of candidate policies, we run CQL with the hyperparameters used in the original paper: 
\begin{itemize}
    \item Regularization coefficient $\in\{0.5,4.0,5.0\}$
    \item Number of gradient steps $\in\{50k, 100k, 150k,\dots,1000k\}$
\end{itemize}

\paragraph{Randomness across multiple runs.}
Similar to the classic RL experiments, we fix the candidate policies and only resample the offline data for OPS. 
In our experiments, we use $5$ different DQN replay dataset and report the average regret with one standard error.

\subsection{Compute}
For Atari experiments, we used NVIDIA Tesla V100 GPUs.
Each run for generating candidate policies, running either FQE or IBES, took less than 6 hours. 
For classic control experiments, we used only CPUs.
Each run for generating candidate policies, running either FQE or IBES, took less than 3 hours.

\end{document}